\documentclass[sigconf]{aamas} 

\usepackage{comment}
\usepackage{todonotes}
\usepackage{algorithm}
\usepackage{algorithmic}
\usepackage{graphicx}
\usepackage{amsthm} 
\usepackage{amsmath}
\usepackage{subcaption}
\usepackage{multirow}
\usepackage{enumitem} 
\usepackage{algorithm}
\usepackage{algorithmic}
\usepackage{amsmath}
\usepackage{graphicx}
\usepackage{xcolor}
\usepackage{cancel}
\usepackage{ifthen}
\usepackage{balance} 

\setlist[itemize]{leftmargin=3.5mm,labelindent=4mm}
\setlist[description]{leftmargin=0mm,labelindent=4mm}
\setlist[enumerate]{leftmargin=5mm,labelindent=-5mm}

\newcommand{\version}{extended}

\newcommand{\colored}{no} 

\newcommand{\black}[1]{\textcolor{black!100}{#1}}

\ifthenelse{\equal{\colored}{yes}}
{
\newcommand{\red}[1]{\textcolor{red!100}{#1}}   
\newcommand{\blue}[1]{\textcolor{blue!100}{#1}}
}
{
\newcommand{\red}[1]{\textcolor{black!100}{#1}} 
\newcommand{\blue}[1]{\textcolor{black!100}{#1}}
}

\doi{KRAP9309}

\makeatletter
\gdef\@copyrightpermission{
  \begin{minipage}{0.2\columnwidth}
   \href{https://creativecommons.org/licenses/by/4.0/}{\includegraphics[width=0.90\textwidth]{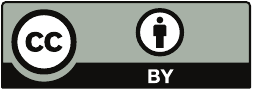}}
  \end{minipage}\hfill
  \begin{minipage}{0.8\columnwidth}
   \href{https://creativecommons.org/licenses/by/4.0/}{This work is licensed under a Creative Commons Attribution International 4.0 License.}
  \end{minipage}
  \vspace{5pt}
}
\makeatother

\setcopyright{ifaamas}
\acmConference[AAMAS '26]{Proc.\@ of the 25th International Conference
on Autonomous Agents and Multiagent Systems (AAMAS 2026)}{May 25 -- 29, 2026}
{Paphos, Cyprus}{C.~Amato, L.~Dennis, V.~Mascardi, J.~Thangarajah (eds.)}
\copyrightyear{2026}
\acmYear{2026}
\acmDOI{}
\acmPrice{}
\acmISBN{}

\acmSubmissionID{fp1772} 

\ifthenelse{\equal{\version}{extended}}
{
\title[Constrained ABA Frameworks]{Constrained Assumption-Based Argumentation Frameworks
(Extended Version with Proofs)}
}
{
\title[Constrained ABA Frameworks]{Constrained Assumption-Based Argumentation Frameworks}
}

\author{Emanuele De Angelis}
\affiliation{
  \institution{CNR-IASI}
  \city{Rome}
  \country{Italy}}
\email{emanuele.deangelis@iasi.cnr.it}

\author{Fabio Fioravanti}
\affiliation{
  \institution{University of Chieti - Pescara}
  \city{Pescara}
  \country{Italy}}
\email{fabio.fioravanti@unich.it}

\author{Maria Chiara Meo}
\affiliation{
  \institution{University of Chieti - Pescara}
  \city{Pescara}
  \country{Italy}}
\email{mariachiara.meo@unich.it}

\author{Alberto Pettorossi}
\affiliation{
  \institution{University of `Tor Vergata'}
  \city{Rome}
  \country{Italy}}
\email{pettorossi@info.uniroma2.it}

\author{Maurizio Proietti}
\affiliation{
  \institution{CNR-IASI}
  \city{Rome}
  \country{Italy}}
\email{maurizio.proietti@iasi.cnr.it}

\author{Francesca Toni}
\affiliation{
  \institution{Imperial}
  \city{London}
  \country{United Kingdom}}
\email{ft@imperial.ac.uk}

\begin{abstract}
Assumption-based Argumentation (ABA) is a well-established form of structured argumentation. 
ABA frameworks with an underlying atomic language are widely studied, but their applicability is limited by a representational  restriction to ground (variable-free)
arguments and attacks built from propositional atoms. 
 In this paper, we lift this restriction and propose a novel notion of \emph{constrained ABA} (\emph{CABA}), whose components, as well as arguments built from them, may include constrained variables, ranging over possibly infinite domains.
 We define non-ground semantics for CABA, in terms of various notions of non-ground attacks.
 We show that the new semantics conservatively generalise standard ABA semantics.
\end{abstract}


\def\FAtt{\mathit{FAtt}}
\def\Arg{\mathit{Arg}}
\def\Args{\mathit{Args}}
\def\MGCArg{\red{\mathit{MGCArg}}}
\def\MGCArgFA{\red{\mathit{MGCArg}(\mathit{FA})}}
\def\TCArg{\red{\mathit{TCArg}}}
\def\CArg{\red{\mathit{CArg}}}
\def\MGCArg{\red{\mathit{MGCArg}}}

\def\GrCInst{\red{\mathit{GrCInst}}}

\def\GroundCArg{\GrCInst(\CArg)}  

\def\IArgs{IArgs}
\def\Att{\mathit{Att}}

\def\contrary{\overline{ \vrule height 5pt depth 3.5pt width 0pt \hskip0.5em\kern0.4em}}
\newcommand{\cabaf}{\ensuremath{\langle \mathcal{L}_c, \, \mathcal{C},  \, \mathcal{R}, \, \ctheory, \,  \mathcal{A},\, \contrary}\rangle}

\newcommand{\abaf}{\ensuremath{\langle \lang, \, \rules, \, \asms,\, \contrary\rangle}}

\newcommand{\abafpone}{\ensuremath{\langle \lang, \, {\mathcal R}', \, {\mathcal A},\, \overline{ \vrule height 5pt depth 3.5pt width 0pt \hskip0.5em\kern0.4em}\rangle}}
\newcommand{\abafp}{\ensuremath{\langle {\mathcal L}', \, {\mathcal R}', \, {\mathcal A}',\, \overline{ \vrule height 5pt depth 3.5pt width 0pt \hskip0.5em\kern0.4em}'\rangle}}
\newcommand{\gabaf}{\ensuremath{\langle {\mathcal L}_{cg}, \, {\mathcal R}_g, \, {\mathcal A}_g,\, \overline{\vrule height 5pt depth 3.5pt width 0pt \hskip0.5em\kern0.4em}^g\rangle}}
\newcommand{\gabafb}{\ensuremath{\langle {\mathcal L}_{cg}, \, {\mathcal R}_g, \, {\mathcal A}_g,$ $ \overline{\vrule height 5pt depth 3.5pt width 0pt \hskip0.5em\kern0.4em}^g\rangle}}
\newcommand{\suB}{\raisebox{4pt}{$\scriptscriptstyle B$}}
\newcommand{\suF}{\raisebox{4pt}{$\scriptscriptstyle F$}}
\newcommand{\abaF}{\ensuremath{\langle {\mathcal L}_{F}, \, {\mathcal R}_F, \, {\mathcal A}_F,\, \overline{ \vrule height 5pt depth 3.5pt width 0pt \hskip0.5em\kern0.4em}{\suF}\rangle}}
\newcommand{\abaB}{\ensuremath{\langle {\mathcal L}_{B}, \, {\mathcal R}_B, \, {\mathcal A}_B,\, \overline{ \vrule height 5pt depth 3.5pt width 0pt \hskip0.5em\kern0.4em}{\suB}\rangle}}
\newcommand{\If}{\ensuremath{\leftarrow}}
\newcommand{\vars}{\ensuremath{
V}}
\newcommand{\termdom}{\ensuremath{
T}}
\newcommand{\preds}{\ensuremath{\mathcal{P}}}
\newcommand{\constr}{\ensuremath{\mathcal{C}}}
\newcommand{\constrfo}{\ensuremath{\mathcal{K}}}
\newcommand{\ctheory}{\ensuremath{\mathcal{C\hspace*{-.5pt}T}}}
\newcommand{\lang}{\ensuremath{\mathcal{L}}}
\newcommand{\langc}{\ensuremath{\mathcal{L}_c}}

\newcommand{\asms}{\ensuremath{\mathcal{A}}}
\newcommand{\rules}{\ensuremath{\mathcal{R}}}

\newcommand{\argur}[3]{#1 \vdash_{#3} #2} 
\newcommand{\argu}[2]{#1 \!\vdash\! #2} 
\newcommand{\ruleset}{\ensuremath{R}}
\newcommand{\tuple}[1]{\ensuremath{\mathsf{#1}}}
\newcommand{\sent}{\ensuremath{s}}

\newcommand{\asm}{\ensuremath{a}}
\newcommand{\asmset}{\ensuremath{S}}
\newcommand{\argset}{\ensuremath\mathit{Arg}}

\newcommand{\A}{\mathcal{A}}
\newcommand{\LL}{\mathcal{L}}
\newcommand{\RR}{\mathcal{R}}
\newcommand{\ABA}{\langle \LL,\RR,\A,\contrary\rangle}
\newcommand{\sABA}{\langle\RR,\A,\contrary\rangle}

\newcommand{\vdashacc}{\ensuremath{\models}}
\newcommand{\entails}{\ensuremath{\sqsupseteq}}
\newcommand{\mpnote}[1]{\noindent \textcolor{brown!100}{(Maurizio)} \textcolor{brown!100}{#1}}
\newcommand{\mpfootnote}[1]{\footnote{ \textcolor{brown!100}{(Maurizio)} \textcolor{brown!100}{#1}}}
\newcommand{\ftnote}[1]{\noindent \textcolor{orange!100}{(Francesca)} \textcolor{orange!100}{#1}}
\newcommand{\cmfootnote}[1]{\footnote{ \textcolor{green!100}{(Chiara)} \textcolor{green!90}{#1}}}
\newcommand{\cmnote}[1]{\noindent \textcolor{green}{(Chiara)} \textcolor{green!90}{#1}}

\newcommand{\adpnote}[1]{\noindent \textcolor{magenta!70}{\rm{(Alberto)}} \textcolor{magenta!70}{\rm{#1}}}

\newcommand{\ff}[1]{\noindent \textcolor{red!60}{\rm{(FF)}} \textcolor{red!60}{\rm{#1}}}

\begin{document}


\pagestyle{fancy}
\fancyhead{}


\maketitle 


\section{Introduction}\label{sec:Intro} 
Assumption-based Argumentation (ABA)~\cite{ABA,ABAbook,ABAtutorial,ABAhandbook} is 
a well-established form of structured argumentation, whereby arguments are deductions for claims, 
supported by rules and assumptions, and attacks between arguments are generated whenever the claim
of the attacking argument is contrary to an assumption supporting the attacked argument.
The semantics of ABA frameworks is 
defined in terms of various notions of 
\emph{acceptable extensions}, amounting to sets of assumptions/arguments that can be 
deemed to defend themselves against attacks, in some sense.

ABA has been advocated as a unifying framework for various forms of non-monotonic reasoning~\cite{ABA}, including logic programming with negation as failure under several semantics (where assumptions are negation as failure literals of the form $\mathit{not}\,p$, and the contrary of $\mathit{not} \, p$ is $p$). 
The logic programming instance of ABA 
finds
 application in several domains, including  legal reasoning~\cite{DBLP:conf/comma/DungT10},
 planning~\cite{DBLP:conf/prima/Fan18a}, healthcare~\cite{DBLP:conf/atal/ZengSCLWCM20,DBLP:journals/argcom/CyrasOKT21,DBLP:conf/ratio/SkibaTW24}, and causal discovery~\cite{russo2024argumentativecausaldiscovery}.
However, 
while other instances of ABA, such as the one to capture default logic, 
rely upon
first-order logic as the underlying deductive system~\cite{ABA}, in the logic programming instance
sentences 
making up rules, assumptions and their 
contraries need to be \emph{propositional} atoms. 

In this paper we propose 
\emph{Constrained ABA $($CABA$)$},
dropping this restriction, as
illustrated 
next in a legal setting.

\ifthenelse{\equal{\version}{extended}}
{
\vspace*{-1mm}
}
{
}

\begin{example}\label{ex:motive}
Let us consider the following CABA rules 
modelling an AI system designed to support a tax department employee to decide whether or not 
a person 
should pay tax.

\vspace*{.5mm}
\noindent R1. $\textit{must\_pay\_tax}(P) \If \textit{income}(P,I),\ I\!\geq\!0,\ \textit{nonexempt}(P)$

\noindent R2. $\textit{exempt}(P) \If \textit{income}(P,I),\ I\!\geq\!0,\ I\!\leq\! 16,000,\  \textit{salary\_income}(P)$

\noindent R3. $\textit{other\_incomes}(P) \If \textit{foreign\_income}(P,F),\ F\!\geq\! 10,000$
\vspace*{.5mm}

These rules serve as templates, with variables $P, I, F$ standing as placeholders for any values, 
in the spirit of argument schemes with critical questions~\cite{waltonbook,waltonhandbook}  and
logic programming (where variables are implicitly universally quantified 
at the front of each rule). 
Rule R1 states 
that anyone is required to pay taxes if he/she has an 
income and has no exemption (formalised 
with assumption $\textit{nonexempt}(P)$, 
with contrary $\textit{exempt}
(P)$).
The exemption 
applies to people having as 
income only their salary (formalised with assumption 
$\textit{salary\_income}(P)$, with contrary $\textit{other\_}\textit{in}\-\textit{comes}(P)$), provided that it is not 
above the 
threshold of 16,000 (in some currency).
Having other sources of income 
 is then specified by 
 Rule R3,
 which states that any income from a foreign country above 10,000 
should be considered 
additional income.
To represent 
facts 
about individuals
in ABA, the rules, assumptions, and contraries should be grounded
.  
Note that 
this may not be possible if the individuals 
to whom the rules apply
are not known up-front.
Without knowing those individuals, while 
grounding may be practicable (yet inefficient) for Rule R2,
it would require 
fixing an upper bound on both the income~$I$ in  Rule R1 and the foreign income~$F$ in the Rule R3, which may not be feasible or desirable in practice.
Constraints, such as $I \!\geq\! 0$, in 
CABA rules make it possible to avoid considering all ground instances altogether.
\end{example}

\vspace*{-1mm}
As illustrated by Example~\ref{ex:motive}, 
CABA frameworks 
may include constrained variables, ranging over possibly infinite domains.
From a practical perspective, CABA frameworks enable the use of specific constraint domains and corresponding solvers within the declarative paradigm of ABA. From a theoretical perspective, this approach offers a unified framework for capturing and analysing the use of constraints alongside ABA.



{\it{Contributions}}. 

\noindent (1)  We define the novel concept of CABA frameworks (Sect.~\ref{sec:CABAF}) and notions 
of non-ground constrained 
arguments (Sect.~\ref{sec:CArg}), attacks (Sect.~\ref{sec:Att}) and extension-based semantics for CABA (Sect.~\ref{sec:sem}), focusing on conflict-free, admissible, and stable extensions.

\noindent (2)
We show (Sect.~\ref{sec:ground}) that the new notions amount to the standard ABA notions after grounding (as well as conservatively generalising 
those in standard ABA).

\noindent (3)
We define (Sect.~\ref{sec:non-ground}) novel notions of extensions for CABA, without explicit reference to grounding 
of standard ABA, and we identify 
conditions on the 
theory of constraints for which they are equivalent to the ones obtained by grounding. These new notions allow us to effectively 
construct finite non-ground CABA extensions also in cases where their ground counterpart are indeed infinite.

\ifthenelse{\equal{\version}{extended}}
{
{This paper is an extended version of \cite{CABA-AAMAS26}
and includes an Appendix with the proofs of all results.}
}
{
\red{The proofs of all results not included in the paper are in the extended version available at \url{http://arxiv.org/abs/2602.13135}.}
\vspace*{2mm}
}

\vspace*{-1mm}

\section{Related work}
Various ABA instances have been considered, to capture existing non-monotonic formalisms~\cite{ABA}. 
Some of these instances, notably corresponding to default logic and circumscription, rely upon first-order logic as the underlying deductive system and can thus be deemed non-ground/non-propositional.
Rather than identifying a new instance of ABA which integrates reasoning with constrained variables, we define a novel framework, CABA, singling out reasoning with variable instances and constraints.
CABA restricts attention to 
atomic languages, in the spirit of logic programming instances of ABA~\cite{ABA}, and integrates a stand-alone constraint solver in the spirit of Constraint Logic Programming (CLP)~\cite{JaL87,JaM94}.

Our CABA frameworks share some representational characteristics with CLP with negation as failure~\cite{JaM94,Clark77}. However, our focus is on the definition of argumentative semantics rather than top-down procedural machinery (e.g. SLDNF-variants with constraints
) as in much CLP research~\cite{JaL87,Stuckey91,Stuckey95}.


Of particular relevance to our work is Answer Set Programming (ASP) with constraints, s(CASP)~\cite{scasp18}, given its use of stable models semantics related to stable extensions for CABA. But s(CASP) too has a procedural, top-down spin. 
Other proposals for integrating reasoning with constraints and ASP, notably~\cite{CabalarKOS16}, 
focus on a semantic angle via the use of the Here and There logic,  but are restricted to stable models whereas we can accommodate other semantics (admissible extensions in this paper).

Frameworks for Abductive Logic Programming with constraints \cite{ciff,sciff} are also related to our work, but they focus on the completion semantics for logic programming~\cite{Clark77,Kunen87}
rather than variants of the stable and admissible extension semantics for ABA~\cite{ABA} as we do.

Within argumentation, argument schemes~\cite{waltonbook,waltonhandbook} are templates for arguments and may thus be seen as a form of non-ground argumentation, but lack a semantics. 
%
Also,
non-ground variants of ABA are used in ABA Learning~\cite{ABAlearning-ecai24}, but as templates, in the spirit of argument schemes. Also, ABA Learning cannot accommodate constraints. 
Finally, 
DeLP~\cite{DeLP} is another form of 
argumentation relying upon logic programming notions, and equipped with a procedural semantics accommodating, like in CABA, rule schemata and constraints therein.
However, differently from DeLP, our CABA adopts
extension-based semantics in the spirit of \cite{Dung_95}.

\ifthenelse{\equal{\version}{extended}}
{
\vspace*{-2.mm}
}
{
\vspace*{-4mm}
}

\section{Background} 
\label{sec:background}

\paragraph{First-order languages.}

A first order language is a set of first order formulas
constructed, as usual
~\cite{Men97}, from given sets of variables, function symbols,
and predicate symbols.
\emph{Atomic languages} are first order languages
whose formulas are \emph{atomic formulas} only. Those atomic formulas, also called \emph{atoms}, are 
made out of predicate symbols applied to terms.

We use capital letters, e.g. $X
$, to denote variables, lower-case letters, e.g. $t
$, to denote terms, 
and sans-serif letters, e.g. $\tuple{X}
$ and 
$\tuple{t}
$, to denote tuples of 
variables and tuples of 
terms, respectively.

A \emph{substitution}
$\vartheta$ is a total function 
mapping $n\ (\geq\!0)$ distinct variables $X_1, \ldots, X_n$ to $n$ terms
$t_1, \ldots, t_n$, and it is denoted by
$\{X_{1}/t_{1},\ldots,$ $X_{n}/t_{n}\}$ or simply $\{\tuple{X}/\tuple{t}\}$. We assume that each $t_i$ is distinct from~$X_i$.
A \emph{variable renaming} (or \emph{renaming}, for 
short)
is a substitution $\{\tuple{X}/\tuple{Y}\}$,
where \tuple{X} and \tuple{Y} are tuples of variables 
and 
\tuple{Y} is a permutation of~\tuple{X}~\cite{Apt90}.
%
%
Given a term~$t$, we denote by $\mathit{vars}(t)$ the set of variables occurring in $t$.
Given a substitution~$\vartheta $ and a term~$t$, an 
\emph{instance of}~$t$ (\emph{via} $\vartheta$), denoted~$t\vartheta $, is the term obtained
by replacing each variable occurrence~$X$ in~$t$ 
by~$\vartheta(X)$.
If 
$\vartheta$ is a renaming, 
then $t\vartheta$
is a \emph{variable renaming} (or \emph{renaming}, for short)
\emph{of} $t$ (\emph{via}~$\vartheta$).
A term~$t$ is  \emph{ground} if $\mathit{vars}(t)\! =\! \emptyset$ and, 
if $\mathit{vars}(t\vartheta)\! =\! \emptyset$, then the 
substitution~$\vartheta$ is said to be
\emph{grounding}.
\ifthenelse{\equal{\version}{extended}}
{
}
{

}
We lift the function $\mathit{vars}$ and 
the notions of `being an instance of', `being a renaming of', 
and `being ground' to tuples, sets, and all the other 
syntactic constructs 
used in this paper.

%

Given a set $V\!=\!\{X_1,\ldots,X_n\}$ of variables and a formula $B$, $\exists _V B$ denotes the 
formula $\exists X_1\ldots\exists X_n\,B$. The \emph{existential closure of} $B$, denoted $\exists(B)$, is 
$\exists_{
V}B$, where $V$ is the set of variables occurring in $B$ outside the scope of a quantifier. The notation 
$\exists _{-V} B$ 
is a shorthand for $\exists _{\mathit{vars(B)} \setminus V} B$.
Similar terminology and  notations are used for $\forall$.

We say that an atomic language $\lang$ is \black{{\emph{predicate closed}}}
if, for every predicate $p$ of arity $n$ in $\lang$ and 
$n$-tuple
\tuple{t} of terms in $\lang$, 
$p(\tuple{t})$ is in $\lang$. 

\vspace{1mm}
\emph{Deductive systems.}
ABA (as well as our constrained version thereof) uses a \emph{deductive system} $(\lang,\rules)$ consisting of 
a language~$\lang$ (not necessarily first-order, in general)
and a 
set~$\rules$ of (inference) rules of the form
$\sent_0 \leftarrow \sent_1,\ldots, \sent_m $ (with $m\! \ge\! 0$ and, for all $i$, with $0\!\leq\! i\! \leq\! m$, 
$\sent_i\! \in\! \lang$);
$\sent_0$ and the sequence $\sent_1,\ldots, \sent_m$ are called 
the {\em head} and the \emph{body} of the rule, respectively.
If $m\!=\!0$ then the rule is represented as  $\sent_0 \If$
and it is called a \emph{fact}.
With an abuse of terminology, we state that every rule $\sent_0 \!\leftarrow\! \sent_1,\ldots, \sent_m$
in~$\rules$ belongs to $\LL$.

\vspace{1mm}
\emph{Theories of constraints.}
We consider 
a \emph{
theory of constraints}, denoted $\ctheory$, which 
is based on
a first order language $\constrfo$
whose  atomic formulas 
$\constr$
are called \black{\emph{atomic constraints} (or \emph{constraints}, for short).}
Constraints are over an algebraic structure with domain $D$.
$\ctheory$ is
equipped with 
a notion of validity: we write $\ctheory \models \varphi$ to denote that the formula 
$\varphi$ of $\constrfo$ is \emph{valid} 
in  $\ctheory$. 
A set $\{c_1,\dots,c_n\}\subseteq 
\constr$ of constraints is said to be \emph{consistent} 
if $\ctheory \models \exists 
(c_1\wedge \ldots \wedge c_n)$.
We assume that $\ctheory$ is a first order
theory with equality, which is interpreted as 
the identity on $D$
and includes the \emph{Clark 
Equality Theory} axiomatising
the identity between ground terms~\cite{Lloyd87}.
Examples of $\ctheory$'s are the 
linear integer arithmetic (LIA) and 
the linear rational arithmetic (LRA).

\emph{Assumption-based argumentation $($ABA$)$.}
An {\em ABA framework} \cite{ABA} (presented here following more recent papers \cite{ABAbook,ABAtutorial,ABAhandbook}) is a 4-tuple \abaf,
where:

\noindent 
$\bullet\ \langle \lang, \rules\rangle$ is a 
deductive system; 

\noindent 
$\bullet\ \asms$ $\subseteq$ $\lang$ is a non-empty
set
of {\em assumptions};\footnote{The non-emptiness requirement
 can always be satisfied by including in $\A$ a \emph{bogus assumption}, with its own contrary, neither occurring elsewhere in the ABA framework. 
For 
conciseness, we will not write this assumption and its contrary explicitly.} 

\noindent \hangindent=3mm
$\bullet\ \contrary$ is a\,total mapping from $\asms$ into
 $\lang$,\,where $\overline{\asm}$ is called the {\em contrary} of $\asm$, 
 for $\asm \in \asms$.

\noindent
An ABA framework is  {\em flat}
if assumptions are not heads of rules.

\vspace*{-1mm}

\begin{example}
\label{ex:simpleABA}
\!$\abaB$ is\,an ABA framework, call it $B$:

\noindent
$\bullet\ \LL_B=\{p(i),\ q(i),\ r(i),\ a(i),\ b(i)\ |\ i\!=\!1,2\}$;

\noindent
$\bullet\ \RR_B\!=\!\{p(\!1) \!\If\!a(\!1), \mathit{q}(\!1)\!\If\! b(\!1),  r(\!1) \!\If,\ 
p(2)\! \If\!a(2), q(2)\!\If \!b(2)\};$

\noindent
$\bullet\ \A_B=\{a(\!1),\ b(\!1),\ a(2),\ b(2)
    \}$; 

\noindent
$\bullet\ \overline{a(\!1)}\suB\!=\!q(\!1)$, \quad  $\overline{a(2)}\suB\!=\!q(2)$, \quad $\overline{b(\!1)}\suB\!=\!r(\!1)$, \quad $\overline{b(2)}\suB\!=\!\mathit{r}(2)$.\

\end{example}

In flat ABA 
frameworks, {\em arguments} are deductions of claims supported by rules and  assumptions, 
and {\em attacks} are directed at the
assumptions which support arguments.  
Formally, 
following \cite{ABAbook
}:

\hangindent=3mm
\hspace*{-4mm}
$\bullet$ \emph{An argument for 
a claim $\sent\! \in \!\lang$ 
supported by $\ruleset \!\subseteq \!\rules$} and $A\! \subseteq\! \asms$, denoted $\argur{A}{\sent}{\ruleset}$,  
is a finite tree 
such that: 
(i) the root is~$\sent$, (ii)~every non-leaf node~$\sent'$ 
has as children 
    all and only the 
    sentences of the body of 
{\color{black}{exactly one rule}} $\mathit{Ri}$ in~\ruleset~
whose head is $\sent'$, or, if $\mathit{Ri}$ is a fact, $s'$ has 
the only child $\mathit{true}$ (representing the empty body), and 
(iii)~every leaf is either 
an assumption in~$A$ or $\mathit{true}$. 
$A$ is the set of all assumptions in the tree
and $R$ is the set of all rules used to construct the tree. 

Note that, for  every assumption $\alpha\!\in\!\asms$, there is an argument \mbox{$\argur{\{\alpha\}}{\alpha}{\emptyset}$,} which is a tree consisting of 
the single node $\alpha$.

\hspace*{-4mm}
$\bullet$ Argument 
$\argur{A_1\!}{\!\!\sent_1}{\ruleset_1}$ 
{\em attacks} argument
$\argur{A_2\!\cup\!\{a\}\!}{\!\!\sent_2}{\ruleset_2}$   
iff
$\sent_1\!=\!\overline{\asm}$. 

In the remainder, 
we will often omit to specify the supporting rules in arguments
(as attacks do not depend on them.)


\begin{example}
\label{ex:simpleABA-arg}
In the ABA framework $B$ of Example~\ref{ex:simpleABA} we have the following set $\Arg_B$ of arguments:
$\big\{\argu{\{a(1)\}}{p(1)}$,\hspace{.5mm}
$\argu{\{b(1)\}}{q(1)}$,\hspace{1mm} $\argu{\emptyset}{r(1)}$,\hspace{.5mm}$\argu{\{a(2)\}}{p(2)}$,\hspace{1mm}$\argu{\{b(2)\}}{q(2)}$,
\hspace{.5mm}$\argu{\{a(1)\}}{a(1)}$,\hspace{1mm}
$\argu{\{a(2)\}}{a(2)}$, \hspace{1mm}
$\argu{\{b(1)\}}{b(1)}$, \hspace{1mm}
$\argu{\{b(2)\}}{b(2)}\big\}$.
Among other attacks, we have that
$\argu{\emptyset}{r(1)}$ attacks
$\argu{\{b(1)\}}{q(1)}$, as $\overline{b(1)}\suB\!=\!\mathit{r}(1)$.


\end{example}

Given a flat ABA framework $F=\abaf$, let $\Arg$ be the set of all 
arguments of $F$
and $\Att$ be the set of all 
attacks of $F$, i.e. $\Att = \{(\alpha,\!\beta)\! \in\!\Arg \times\! \Arg \mid \alpha$ attacks $\beta\}$.
Then $(\Arg,\Att)$ is an \emph{Abstract Argumentation} (AA) framework \cite{Dung_95} and the standard semantics 
for the latter 
can be used to determine  the semantics for the ABA framework
\cite{ABAtutorial}.\footnote{ABA semantics were originally defined in terms of sets of assumptions and attacks between them \cite{ABA}, but can be reformulated, for flat ABA frameworks, 
in terms of sets of arguments and attacks between them (see~\cite{ABAtutorial}), as given here. }
For example, we say that: for all $\Gamma\!\subseteq\! \Arg$,

\noindent
$\bullet\ \Gamma$ is a \emph{conflict-free extension} iff $\not\exists\, \alpha,\beta \!\in \!\Gamma$ 
s.t. 
$(\alpha,\beta) \!\in \!\Att$;

\noindent\hangindent=3mm
$\bullet\ \Gamma$ is an \emph{admissible extension} iff (i)~$\Gamma$ is conflict-free, and (ii)~$\forall \alpha\!\in\! \Gamma$, 
$\forall \beta \!\in \!\Arg$ if $(\beta,\alpha)\!\in\! 
\Att$, then $\exists\, \gamma \!\in \!\Gamma$ s. t. 
$(\gamma,\beta) \!\in \!\Att$ (i.e. $\Gamma$ 
``attacks'' every argument that 
attacks any of its arguments);

\noindent\hangindent=3mm
$\bullet\ \Gamma$ is  a
\emph{stable
extension} iff (i)~$\Gamma$ is conflict-free, and 
(ii)~$\forall \beta \!\in
\!\Arg\!\setminus\!\Gamma,$ $\exists\, \alpha \!\in \!\Gamma$ s.t. 
$(\alpha,\beta) \!\in \!\Att$ (i.e. $\Gamma$ 
``attacks'' all arguments it does not contain).

\begin{example}
\label{ex:simpleABA-ArgsAtt}
The AA framework of 
the ABA framework $B$ of Example~\ref{ex:simpleABA}
is $(\Arg_B,\Att_B)$ with $\Arg_B$ as given in Example~\ref{ex:simpleABA-arg} and 
$\Att_B=\bigl\{\bigl(\argu{\emptyset}{\!r(1)},\argu{\{b(1)\}}{q(1)}\bigr), 
\bigl(\argu{\emptyset}{r(1)},\argu{\{b(1)\}}{b(1)}\bigr)\bigr\} 
\cup \bigl\{\bigl(\argu{\{b(X)\}}{q(X)},\argu{\{a(X)\}}{p(X)}\bigr),
 \bigl(\argu{\{b(X)\}}{q(X)},\argu{\{a(X)\}}{a(X)}\bigr)\ $  $|\ 
X \!=\! 1,2\bigr\}$. Then, the set 
$\bigl\{\argu{\emptyset}{r(1)}$, $\argu{\{a(1)\}}{p(1)}$, $\argu{\{a(1)\}}{a(1)},$
$\argu{\{b(2)\}}{q(2)}, \argu{\{b(2)\}}{b(2)}\bigr\}$ of arguments is the only stable extension of $B$. 
That extension is also admissible. 
The claims 
in that stable extension 
make out the ``model'' $\{r(\!1), p(\!1), a(\!1), $ 
$q(2), b(2)\}$~\cite{ABAtutorial}.
\end{example}
%



\section{
CABA Frameworks}\label{sec:CABAF}
This section introduces a novel 
formalism that integrates a generic constraint-based computational mechanism into ABA frameworks. 

\begin{definition}\label{def:frameworkFc}
A {\em Constrained Assumption-Based Argumentation $($CABA$)$ framework}
is a $6$-tuple
$\cabaf$, where:

\noindent
$\bullet\ \langc$ is an  atomic language;

\noindent
$\bullet\ \constr \subseteq \langc$ is a 
set of  atomic 
constraints; 

\noindent\hangindent=3mm
$\bullet\ \rules$ is a set of \emph{$($inference$)$ rules} of the form 
$\sent_0 \leftarrow \sent_1,\ldots, \sent_m $,  where $\sent_0\! \in\! \langc \!\setminus\! \constr$, $m\!\geq\!0 $, and,
 for all $i$, with $1\!\leq\! i\! \leq\! m$, $s_i\! \in \!\langc$;
 
\noindent
$\bullet\ \ctheory$\hspace{1pt}is\,a\,theory\hspace{1pt}of\,constraints\,whose\,set\,of\,atomic\,constraints\,is\,$\constr$;

\noindent
$\bullet\ \asms \subseteq$ $\langc \!\setminus\! \constr$ 
is a non-empty set of {\em assumptions}; 

\noindent\hangindent=3mm
$\bullet\ \contrary$ is a 
total mapping 
from $\asms$ into
 $\langc \!\setminus\! \constr$, s.t. 
 if $p(\tuple{t}_1)$, $p(\tuple{t}_2) \!\in \!\asms$, then
 there exists a predicate symbol, say $\mathit{cp}$, with 
$\overline{p(\tuple{t}_1)}\!=\!\mathit{cp}(\tuple{t}_1)$ and
$\overline{p(\tuple{t}_2)}\!=\!\mathit{cp}(\tuple{t}_2)$;
for all $a\! \in\! \asms$, $\overline{a}$  
is called the {\em{contrary}} of $a$.

\noindent The atomic languages 
$\langc$, 
$\constr$, and $\asms$ are all predicate closed.
\end{definition}

In the remainder, unless otherwise specified, 
we assume 
we are given a CABA framework $F_c=\cabaf$.

\noindent 

Note that, differently from standard ABA, in CABA the rules may be non-ground. 
When rules have variables, they represent rule schemata whose variables 
can be instantiated by terms. 
For example, given  rule $R$: $~p(X)\!\If\! X\!<\!3, a(X)$ and the substitution $\vartheta\!=\!\{X/2\}$, we get the ground rule  $R\vartheta$: 
$~p(2)\If 2\!<\!3, a(2)$. 


Here, we focus on \emph{flat} CABA frameworks, that is, we stipulate that 
no assumptions occur in heads of rules.

\begin{example}
\label{ex:simplenew}
Let $V = \{X,Y,\ldots\}$ be a countable set of variables, $\mathbb Q$ be the set of rationals, and 
$\termdom$ be the set of terms built out of $V \cup \mathbb Q$  and  
binary function symbols $+$ and $-$ only. 
Then, 
the following 6-tuple $\cabaf$ is a CABA framework, call it $\mathit{FA}$\,:

\noindent\hangindent=3mm
$\bullet\ \langc=\{p(t), s(t), a(t_1,t_2),  b(t), \mathit{ca}(t_1,t_2), \mathit{cb}(t) \mid t,t_1,t_2\! \in\! \termdom\}\cup \constr$ 

\noindent\hangindent=3mm
$\bullet\ \constr = \{t_1\!<\! t_2$, $t_1\!\leq\!t_2$,\ $t_1\!=\!t_2,\ t_1\!\not =\!t_2,\ t_1\!\geq\! t_2,\ t_1\!>\! t_2 \mid t_1, t_2\in \termdom\}$;

\noindent\hangindent=3mm
$\bullet\ \RR=\{R1, R2, R3, R4, R5\}$,  where:

$\begin{array}{l@{\hspace{-3mm}}l@{\hspace{1.5mm}}ll@{\hspace{1.5mm}}ll}
   & R1. &  p(X)\If X\!<\!1,\ a(X, Y),\ b(X),\ s(Y)\\
& R2.   &  p(X)\If a(Y,X),\ cb(Y) &R3.   &  s(Y) \If Y\!>\!0\\

& R4.   &  \mathit{ca}(X, Y)\If X\!<\!5,\ Y\!>\!3 & R5.   &  \mathit{cb}(Y)\If Y\!<\!10\\
\end{array} $

\noindent\hangindent=3mm
$\bullet\ \ctheory$ is 
the theory of linear rational arithmetic (LRA);

\noindent\hangindent=3mm
$\bullet\ \A=\{a(t_1, t_2) \mid t_1, t_2\in \termdom\} \cup \{b(t) \mid t\in \termdom\}$;

\noindent\hangindent=3mm
$\bullet$ for all $t, t_1, t_2\in\termdom$, 
$\overline{a(t_1, t_2)}=ca(t_1, t_2)$ and $\overline{b(t)}=cb(t)$.
\end{example}

With abuse of notation, we will write as
$\sent_0 \If C, \sent_1, \ldots, \sent_m$
any rule of the form $\sent_0 \If c_1,\ldots,c_k, 
\sent_1, \ldots,
\sent_m$, where $k\!\geq\! 0$, $m\!\geq\! 0$, 
$C\!=\!\{c_1,\!\ldots,\!c_k\}\!\subseteq\! \constr$, and 
$\{\sent_1, 
\ldots, \sent_m\}\! \subseteq \langc \!\setminus \constr$. 
Without\,loss\,of generality, we assume that every 
rule in~$\rules$ is written in the
\emph{normalised form}
 $p(\tuple{X}_0) \!\If\! C, p_1(\tuple{X}_1), \ldots, 
 p_m(\tuple{X}_m)$
 where each tuple $\tuple{X}_i$
 is made of 
 distinct variables, at the price of adding equality constraints.
For example, we write\,$p(X,X,4+1)\!\If\! \mbox{$X\!\!<\!3$},\ a(7)$ as 
$p(X, Y,Z)\!\If\! X\!=\!Y$,\ $\mbox{$Z\!=\!4\!+\!1$},\   \mbox{$X\!\!<\!3$},\ \mbox{$U\!=\!7$},\ a(U)$.

\vspace*{-1mm}

\subsection{Relating CABA and standard ABA}

Theorem~\ref{thm:grounding} below 
shows how to view 
(flat) CABA frameworks as standard 
(flat) ABA frameworks.
Indeed,  we can define a grounding procedure, called 
$\mathit{Ground}$, for 
transforming 
any CABA framework~$F_c$ (see
Definition~\ref{def:frameworkFc}) into an equivalent
ABA framework. 

\begin{definition}\label{def:GroundFc}
$\mathit{Ground}(F_c)$ is the  4-tuple $\gabafb$ 
whe\-re\,$\vartheta$ is\,a grounding substitution: 

\noindent
$\bullet\ \LL_{cg}\!=\!\{s\! \in\!\langc 
\!\mid\!s$\,is \textit{ground}\,\};

\noindent
$\bullet\ {\RR}_g \!= \! \{R\vartheta \in\!
\LL_{cg} \mid$ $R \in \RR\}\ \cup\ \{c\If~ \mid c\! \in\! \LL_{cg}\cap \constr,\ \ctheory\!\models\! c\};$

\noindent
\makebox[40mm][l]{$\bullet\ \A_g\! =\! \{\asm\!\in\! \A \mid\! \asm$ is $\mathit{ground}\};$}

\noindent
$\bullet\ \overline{\asm}^g = \overline{\asm}$, for $\asm\!\in\!\A_g$

\noindent We denote by $\mathit{Arg}_c$ the set of all 
arguments in $\mathit{Ground}(F_c)$.
\end{definition}

\vspace*{-2mm}
\begin{theorem} \label{thm:grounding}
$\mathit{Ground}(F_c)$ is an ABA framework.
\end{theorem}

\vspace*{-1mm}

This result enables us to use the ABA semantics to reason about CABA frameworks by translating their components into the ground 
counterparts.
However, this may require an expensive, possibly infinite, grounding. 
\emph{Constrained arguments}, 
given next, may avoid~this.

Trivially,
the converse of Theorem~\ref{thm:grounding} holds, in that every standard ABA framework $\abaf$ where $\LL$ and $\A$ are
predicate closed 
can be seen as the CABA framework 
$\langle \LL,\emptyset,\RR,\emptyset,\A,\contrary\rangle$.
Thus, CABA frameworks are a conservative generalisation of ABA frameworks.
We will refer to these CABA frameworks, corresponding to
ABA frameworks, as \emph{ABA-as-CABA frameworks}.

\section{Constrained Arguments in CABA}\label{sec:CArg}

In this section, we introduce 
the notion of a \emph{constrained argument} 
in a CABA framework $F_c$
. 
Within that general notion, we identify the special case of a
\emph{most general constrained argument}.
During the construction of arguments, 
we 
use rules in normalized form  
that have been {renamed apart} (that is, renamed using fresh new variables), thus preventing 
undesirable clashes of names. 

\begin{definition}
\label{def:TCArg-MGCArg}
 Let $C\!\subseteq\! \constr$ be a {\it{consistent}} set of constraints, \mbox{$A\!\subseteq\!\asms$} a set of assumptions, and $\ruleset\! \subseteq\! \rules$ a set of rules.
 
\noindent
$\bullet\ $\hangindent=3mm
A \emph{tight constrained argument} for claim $\sent \in \mbox{$\lang_c\!\setminus\!\constr$}$
\emph{supported by} $C\cup A$ \emph{and}~$\ruleset$,
denoted $\argur{C\cup A}{\sent}{\ruleset}$,
is a finite tree such that: (i)~the root is $s$, (ii)~every non-leaf node 
$\sent'=p(\tuple{t})$ has as children all and only 
the instantiated atoms 
$s_1\vartheta,\ldots, s_m\vartheta$ 
exactly one renamed apart 
rule $\mathit{Ri}$:
$p(\tuple{X}) \leftarrow s_1,\ldots, s_m$ in~\ruleset{} 
with $\vartheta\!=\!\{\tuple{X}/\tuple{t}\}$, or, if~$\mathit{Ri}$ is a fact, 
$\sent'$ has the only child $\mathit{true}$, and 
(iii)~every leaf is either a  constraint in $C$  or an assumption in $A$ or 
$\mathit{true}$. 
$C$ and $A$ are the sets of all 
constraints 
and assumptions, respectively, occurring in the tree, and $R$ is the set of all rules used to construct the tree.

\noindent
$\bullet\ $\hangindent=3mm A \emph{most general constrained argument} is a tight constrained argument  whose claim is an atom of the form $p(\tuple{X})$.

\noindent \hangindent=0mm 
We denote by 
$\TCArg$ and $\MGCArg$ the set of all tight 
and
most general constrained arguments, respectively,
of 
$F_c$.
\end{definition}

From Definition~\ref{def:TCArg-MGCArg},  it follows  that $\MGCArg 
\subseteq \TCArg$. Actually, we have that $\MGCArg \!\subset\!\TCArg$,
as some tight constrained arguments are not  
most general, as illustrated next. 

\begin{example}\label{ex:TCArg-MGCArg}
\label{ex:tightconstrArg1}
Let us consider the CABA framework 
$\mathit{FA}$ of 
Example~\ref{ex:simplenew}. The following
are tight constrained 
arguments:
$(i)~\argur{\{\mbox{$Y\!<\!10$},$ $a(Y,0)\}}{\!p(0)}{\{R2,R5\}\!}$,\  
 $(ii)\argur{\{0\!<\!1,Y\!>\!0,a(0,Y),b(0)\}}{\!p(0)}{\{R1,R3\}\!}$, 
$(iii)~\argur{\{4\!>\!0\}}{\!s(4)}{\{R3\}\!}$, \ 
and $(iv) \argur{\{X\!\!<\!\!10\}\!}{\mathit{cb}(X)}{\{R5\}\!\!}$. Argu\-ment~$(iv)$ is a most general constrained argument; the others are not.
\end{example}

The following definitions introduce 
notions of  \emph{constrained arguments} and 
\emph{constrained instances}
thereof.


\begin{definition}\label{def:constrained-arg}
$\argur{C'\cup A'}{\sent'}{R}$
is a \black{\emph{constrained argument}} 
in $F_c$ 
if there exist a \emph{tight} constrained argument
$\argur{C\cup A}{\sent}{\ruleset}$ in $F_c$, 
a substitution~$\vartheta$, and a
set $D\!\subseteq\constr$ of 
constraints such that: 
$(i)$~$C'\!=\!(C\vartheta)\cup D$, $A'\!=\!A\vartheta$,
$s'\!=\!s\vartheta$, and $(ii)$~$C'$ 
is consistent.
We denote by 
$\CArg$ the set of all constrained arguments 
of 
$F_c$.

\end{definition}

%
From Definitions~\ref{def:TCArg-MGCArg} and~\ref{def:constrained-arg}, 
we get  that $\TCArg \!\subseteq\!\CArg$.
Actually, we have that $\TCArg \!\subset\!
\CArg$,
as some constrained 
arguments are not  
tight (see Example~\ref{ex:tightconstrArg1new} below). 

\begin{definition}\label{def:instance}
The constrained argument $\alpha'\!=\!\argur{C'\cup A'}{\sent'}{R}$ is a 
\emph{constrained instance $($via $\vartheta$ and~$D$)} 
of the constrained argument~$\alpha\!=\!\argur{C\cup A}{\sent}{R}$ if: 
$(i)$~$C'\!=\!(C\vartheta)\cup D$, $A'\!=\!A\vartheta$,
$s'\!=\!s\vartheta$, and $(ii)$~$C'$ 
is consistent.
We
say that $\alpha'$ is an \black{\emph{instance 
$($via~$\vartheta$$)$}} of $\alpha$ if $\alpha'$  is a 
constrained instance (via $\vartheta$ and~$D$)
of  $\alpha$ with $D\!=\!\emptyset$.

\end{definition}





In what follows, as for arguments of  ABA frameworks,
also for constrained arguments of  CABA frameworks we will often 
omit to specify the set $R$ of rules involved. 
For instance, we will write $\argu{\{\mbox{$Y\!<\!10$}, a(Y\!,X)\}}
{p(X)}$ instead of 
$\argur{\{\mbox{$Y\!<\!10$}, a(Y\!,X)\}}{p(X)}{\{R2, R5\}}$.

\begin{example}
\label{ex:tightconstrArg1new}
\!For the CABA framework 
$\mathit{FA}$ of 
Example~\ref{ex:simplenew}, 
$\argu{\{\mbox{$8\!<\!10$}, \linebreak 
\mbox{$8\!<\!12$}, a(8,0)\}}{p(0)}$
is a constrained instance 
(via $\vartheta\!=\!\{Y/8\}$ and 
$D\!=\!\{8\!<\!12\}$) of 
the tight constrained argument 
$\alpha = \argu{\{\mbox{$Y\!<\!10$},$ $a(Y,0)\}}{p(0)}$.
Also,
$\argu{\{Y\!<\!10, Y\!=\!7,$ $Y\!>\!5,$ $a(Y,0)\}}{p(0)}$
is a constrained instance of the non-tight constrained argument
$\argu{\{Y\!<\!10,$ $Y\!=\!7,$ $a(Y,0)\}}{p(0)}$ which is, in turn, an instance of $\alpha$. 
\end{example}

\begin{example}
\label{ex:simple7newMG}
Let $M$ be the set $\MGCArgFA$  of all most general constrained arguments
of the CABA framework $\mathit{FA}$ 
of Example~\ref{ex:simplenew}. $M$
is obtained by considering all variable renamings of the constrained arguments in the set 
$\widetilde{M}=\{\alpha_1, 
\ldots, \alpha_7\}$,  where:


\noindent
$\begin{array}{l@{\hspace{1mm}}l@{\hspace{-12mm}}l@{\hspace{1mm}}ll}
    \alpha_1. &\argur{\{X\!<\!1, Y\!>\!0,\ a(X,Y),\ b(X)\}}{p(X)}{\{R1,R3\}}; \\
    \alpha_2. &\argur{\{Y\!<\!10,a(Y,X)\}}{\!p(X)}{\{R2,R5\}}; &
    \alpha_3. &\argur{\{Y\!>\!0\}}{s(Y)}{\{R3\}};
\end{array}$\\[-1mm]

\noindent
$\begin{array}{l@{\hspace{1mm}}l@{\hspace{5.5mm}}l@{\hspace{1mm}}l@{\hspace{2mm}}}
     \alpha_4. &  \argur{\{X\!<\!5,\ Y\!>\!3\}}{ca(X,Y)}{\{R4\}}; &
           \alpha_5. &  \argur{\{X\!<\!10\}}{cb(X)}{\{R5\}};
     \\
     \alpha_6. &  \argur{\{a(X,Y)\}}{a(X,Y) }{\emptyset};  &
     \alpha_7. &  \argur{\{b(X)\}}{b(X)
     }{\emptyset}. 
\end{array}$



Note that, in particular,
the most general constrained argument
$\argur{\{U\!\!<\!1, \mbox{$Z\!>\!0$}, a(U,\!Z),  b(U)\}\!}{\!p(U)}
{\{R1,R3\}}$ is an element of $M$ and it is
derived from $\alpha_1$ by 
the variable renaming
$\{X\!/U,Y\!/Z,U\!/X,Z\!/Y\}$.

\end{example}

For ABA-as-CABA frameworks, constrained arguments are standard arguments in ABA, and thus, again, CABA  is a conservative generalisation of ABA. Next, we explore the relation between CABA and ABA arguments for generic CABA frameworks.

\subsection{
Mapping Arguments in CABA 
to ABA}


The following proposition establishes that every tight constrained argument can be obtained as an 
instance of a most general constrained argument. 
As a consequence,
the most general constrained arguments can provide a 
foundation for reconstructing the entire argumentative structure of a CABA framework.

\begin{proposition}
\label{prop:equalmostconstr}
    \!$\argur{C\cup A}{p(\tuple{t} )}{\ruleset}\in \TCArg$ iff there exists 
    $\argur{C'\cup A'}{p(\tuple{X} )}{\ruleset}\!\in\! \MGCArg$ 
         such that for $\vartheta\!=\!\{\tuple{X} /\tuple{t} \}$, we have that
         $C\!=\!C'\vartheta$, $A\!=\!A'\vartheta$, and $C'\vartheta$ is consistent.
\end{proposition}

\begin{definition}
\label{def:consistentgrounding-final}
\!For any 
$\alpha\!\in \!\CArg$, $\GrCInst(\alpha)$ 
is the set of all con\-strained instances of $\alpha$ that are ground.
For any set $\Gamma\!\subseteq\!\CArg$
of constrained arguments, 
$\GrCInst(\Gamma)$ 
is the set $\bigcup_{\alpha\in\Gamma} \GrCInst(\alpha)$.

\end{definition}

The following corollary relates the sets of ground instances 
of constrained arguments and most 
general constrained arguments.\\[-5mm]

\begin{corollary}
\label{cor:groundArg=CGround(M)}
    $\GroundCArg=\GrCInst(\MGCArg)$.
\end{corollary}

The following definition 
is based on the fact that, for any ground constrained argument $\argu{\{c_1,\!\ldots,c_m\}\cup A }{\sent}$ (with $m\!\geq\! 0$), we have~that 
$\ctheory \!\models\! (c_1\!\wedge\! \ldots\!\wedge\! 
c_m)$.
By abuse of notation, for all ground 
atoms $p(\tuple{t})$, $p(\tuple{t'})$ in 
$\langc$, we write
$\ctheory\! \models\! p(\tuple{t})\! 
\leftrightarrow\! p(\tuple{t'})$ 
iff $\ctheory \models \tuple{t}\! =\! 
\tuple{t'}$. 
Given two ground conjunctions\,$C_1$\,and\,$C_2$ of atoms, 
we write $\ctheory\!\models\! C_1\! \leftrightarrow\! C_2$ 
iff  $(i)$~$\forall a_1 \!\in\! C_1$ $,\exists a_2 \!\in\! C_2$ such that $\ctheory 
\!\models\! a_1\! \leftrightarrow\! a_2$, and $(ii)$ 
the same as~$(i)$ by interchanging $C_1$ and $C_2$.

\begin{definition}
\label{def:mod-groundConstraints}
\!\!The\,ground\,constrained\,arguments\,
$\alpha\!=\!\argu{C\!\cup\! A}{\!\sent}$ and 
$\alpha'\!=\!\argu{C'\!\cup\! A'\!}{\!\sent'}$\,are \emph{equal} 
(written $\alpha\!=\!\alpha'$) 
\textit{modulo ground} \textit{con\-straints} 
if\,$(i)\,\ctheory \!\models\! \bigwedge \!\{a \!\mid\! a\!\in\! A\} 
\!\leftrightarrow\! \bigwedge \!\{a\! \mid\! a\!\in\! A'\}$,\,and\!
$(ii)\,\ctheory \!\models\!s\!\leftrightarrow\!s'$.
\end{definition}

For instance,  $\argu{\{0\!<\!10,a(0,2+1),a(0,1+2)\}}{p(4-1)}$ $=$
$\argu{\{a(0,3)\}}{p(3)}$
modulo ground constraints.
In the sequel, a ground constrained argument
$\argu{C \cup A}{\sent}$ will
also be written  as 
$\argu{A'}{\sent'}$, where $A'$ and~$s'$ satisfy Conditions~$(i)$ and $(ii)$, respectively, of 
Definition~\ref{def:mod-groundConstraints}.
Indeed, 
the elimination of consistent, ground constraints 
preserves equality modulo ground constraints.



\begin{example}
\label{ex:simplenewgr-final}
Let $M$ be the set $\MGCArgFA$ 
of the most general constrained 
arguments of 
$\mathit{FA}$ of Example~\ref{ex:simplenew}.
The set $\GrCInst(M)$ of the ground
constrained instances of $M$
is the following, modulo ground constraints.
The subscripts $\alpha_i$'s relate the associated sets 
to the arguments of~$\widetilde{M}$ shown in Example~\ref{ex:simple7newMG}. We assume that $n,m\!\in\!\mathbb Q$.

\hspace*{-4mm}
$\begin{array}{llll}
    \GrCInst(M)= \{ \argu{\{a(n, m), b(n)\}} {p(n)}
         \mid n\!<\!1,\ m\!>\!0\}_{\alpha_1} \hspace*{-2mm}& \hspace{5mm}\cup & \\[-.5mm]
\end{array}$

\hspace*{-5mm}
$\begin{array}{l@{\hspace{1mm}}l@{\hspace{2mm}}l@{\hspace{2mm}}lll}
         \{ \argu{\{a(m,n)\}}{p(n)} \mid  m\!<\!10\}_{\alpha_2} & \cup &
         \{ \argu{\{\, \}}{s(n)} \mid n\!>\!0\}_{\alpha_3} & \cup &\\
         \{\argu{\{\,\}}{ca(n,m)} \mid n\!<\!5,\ m>3\}_{\alpha_4}&  \cup\  &
     \{ \argu{\{\,\}}{cb(n)} \mid n\!<\!10 \}_{\alpha_5} & \cup &\\
     \{ \argu{\{a(n,m)\}}{a(n,m)} \}_{\alpha_6}  & \cup \ & 
     \{ \argu{\{b(n)\}}{b(n)} \}_{\alpha_7}.& 
\end{array}$
\end{example}
\vspace*{-1mm}

The following result establishes a correspondence between the arguments in the  CABA framework $F_c$ and those in its ground ABA counterpart $\mathit{Ground}(F_c)$ (see Definition~\ref{def:GroundFc}). 

\vspace*{-1mm}

\begin{theorem}
   \label{theo:GroundCABAvsABA} 
$\GroundCArg\!=\! \mathit{Arg}_c$  modulo ground constraints.
\end{theorem}

\vspace*{-1mm}
We can lift the correspondence stated by this theorem to
 sets of arguments (and extensions, 
as defined below), by first introducing the following  
equivalence relation which extends the equivalence based on
equality modulo ground constraints. 

\begin{definition}
\label{def:equiv}
The \emph{equivalence relation $\equiv$} over  pairs of
subsets of $\CArg$ is the smallest equivalence relation satisfying the following:

\hangindent=5mm
\noindent $(1)$
\label{equiv:mod-groundconstr} \emph{Ground Constraints}. 
For any two ground constrained arguments 
$\alpha_1$\,and\,$\alpha_2$,\,if\,$\alpha_1\!\!=\!\alpha_2$ 
modulo\,ground\,constraints,\,then\,$\{\alpha_1\!\}\!\equiv\!\{\alpha_2\}$.


\hangindent=5mm
\noindent $(2)$ \label{equiv:ground-congr-instances}
\emph{Ground Instances}. 
For every $\alpha\!\in\!\CArg$,
$\{\alpha\} \equiv$ $\GrCInst(\alpha)$. 


\hangindent=5mm
\noindent $(3)$ \emph{Congruence for Union.}
    \label{equiv:congr-union} For all sets 
    $\Gamma, \Delta_1,$ $\Delta_2\subseteq 
    \CArg$, 
    if \mbox{$\Delta_1 \!\equiv\!\Delta_2$,} then $(\Gamma\! \cup\! \Delta_1) \equiv (\Gamma\! \cup\! \Delta_2)$. 


    
\end{definition}

The following theorem establishes a key 
correspondence between the equivalence of 
sets of constrained arguments and the equality 
of the sets of their ground 
instances.

\begin{theorem}  
 \label{theo:groundequiv}
For every $\Gamma, \Delta \subseteq \CArg$, 
$\Gamma \equiv \Delta$ iff\, 
$\GrCInst(\Gamma) = \GrCInst(\Delta)$ modulo ground constraints. 
\end{theorem}

From  Theorem~\ref{theo:GroundCABAvsABA} and Theorem~\ref{theo:groundequiv}, we get the following.

\begin{corollary}
\label{cor:equivinground}
Let $\Gamma,\Delta\! \subseteq\! \CArg$ and let $ \Gamma_c, \Delta_c\! \subseteq\! \mathit{Arg}_c$ be 
the sets of arguments in $\mathit{Ground}(F_c)$ (see Definition~\ref{def:GroundFc}) which correspond to 
$\GrCInst(\Gamma)$ and $\GrCInst(\Delta)$, respectively, as established by 
Theorem~\ref{theo:GroundCABAvsABA}. Then $\Gamma \equiv \Delta$ iff\, 
    $\Gamma_c = \Delta_c$ modulo ground 
    constraints.
\end{corollary}

As a consequence of  Corollaries~\ref{cor:groundArg=CGround(M)} and~\ref{cor:equivinground},  
for any $\Gamma\! \subseteq\! \CArg$, we may use $\GrCInst(\Gamma)$ to
refer both to: (i)~the set of 
ground constrained instances of $\Gamma$, and (ii)~the corresponding set of 
arguments in  
$\mathit{Ground}(F_c)$. 

Now we present some properties of the 
equivalence relation $\equiv$.

\begin{proposition}\label{prop:equivproperties}
\hangindent=5mm
\noindent $(1)$ \label{equiv:mod-renaming} \emph{Renaming}. For any constrained argument~$\alpha$ 
and variable renaming~$\rho$,
$\{\alpha\} \equiv \{\alpha \rho\}$.

\hangindent=5mm
\noindent $(2)$ \label{equiv:mod-mostgeneral} \emph{Generalization of Claims and Assumptions}. 
Let $\alpha=\argu{C \cup A}{p(\tuple{t})}$ 
be a constrained argument and let $\tuple{X}$ be a tuple of new distinct variables.
Then 
$\{\alpha\}\!\equiv\!\{\argu{C \cup 
\mbox{$\{\tuple{X}\!=\!\tuple{t}\}$}\cup A}{p(\tuple{X})}\}$.
Similarly, $\{\argu{C \cup A\cup \{p(\tuple{t})\}}{s}\}\!\equiv\!\{\argu{C \cup 
\mbox{$\{\tuple{X}\!=\!\tuple{t}\}$} \cup A\cup \{p(\tuple{X})\}}{s}\}$.



\hangindent=5mm 
\noindent $(3)$
\label{equiv:disj-constr} 
\emph{Constraint Equivalence}. 
Let $C, D_1,$ and $D_2$ be finite sets of consistent
constraints. 
If $\ctheory \models$ $\forall\big(\exists_{-V}
\big(\bigwedge \{c\!\mid\! c\!\in\! C\}\big) \leftrightarrow 
\big(\exists_{-V}\big(\bigwedge \{c\!\mid\! c\!\in\! D_1\}\big) \vee$\! $\exists_{-V}\big(\bigwedge \{c\!\mid\! c\!\in\! D_2\}\big)\big)\big)$, where
 $V\!=\!\mathit{vars}(\{A,s\})$,
then 
      $\{\argu{C \cup A}{\sent}\} \equiv \{\argu{D_1 \cup A}{\sent},$ $\argu{D_2 \cup A}{\sent}\}$.

\end{proposition}




By Point~(3) of Proposition~\ref{prop:equivproperties}, for \mbox{$D_1\!=\!D_2\!=\!D$}
we have that, for all finite sets $C$ and $D$ of constraints, if 
$\ctheory \models \forall\bigl(
\mbox{$\bigwedge \{c\!\mid\! c\!\in\! C\}$}$ $\!\leftrightarrow\! 
\bigwedge \{c\!\mid\! c\!\in\! D\} \bigr)$,
then 
$\{\argu{C \cup A}{\sent}\} \!\equiv\! \{\argu{D \cup A}{\sent}\}$.

In Examples~\ref{ex:simple7newMG} and \ref{ex:simplenewgr-final}, (i)~by
Points~(1) and~(2) 
of Definition~\ref{def:equiv}, we have that $\{\alpha_1\} \equiv 
\{\argu{\{a(n, m), b(n)\}} {p(n)}
     \mid \mbox{$n\!<\!1,$}\ \mbox{$m\!>\!0$}\}$,
and (ii)~by 
Definition~\ref{def:equiv}, we have that $M\equiv \widetilde{M}\equiv 
\GrCInst(M)$. 




\begin{example}[Ex.~\ref{ex:simple7newMG} contd.]
\label{ex:equivalence}
Let $\widetilde N$ be 
the set~$\widetilde M$ where 
$\alpha_2$ has been replaced by:

\noindent
$\begin{array}{l@{\hspace{1.mm}}ll}
        \alpha_{2,1}\!: & \argu{\{Y\!<\!10,\ Y\!\geq\!5,\   a(Y,X)\}\ }{\ p(X)};
\end{array}$

\noindent
$\begin{array}{l@{\hspace{1.5mm}}ll}
     \alpha_{2,2,1}\!: & \argu{\{Y\!<\!10,\  Y\!<\!5,\  X\!>\!3,\  a(Y,X)\}\ }{\ p(X)};
\end{array}$

\noindent
$\begin{array}{l@{\hspace{1.5mm}}ll}
      \alpha_{2,2,2}\!:& \argu{\{Y\!<\!10,\ Y\!<\!5,\ X\!\leq\!3,\   a(Y,X)\}\ }{\ p(X)}\\[1mm]
\end{array}$

\noindent 
Now, 
by  Point~(3)  of 
Proposition~\ref{prop:equivproperties} with $V\!=\!\{X,\!Y\}$, and 
Point~(3) 
of Definition~\ref{def:equiv}, from 
$\ctheory \models \forall (Y\!\geq\!5 \vee Y\!<\!5)$ and 
$\ctheory \models \forall \mbox{$(X\!>\!3$} \vee X\!\leq\!3)$,
we get $\widetilde M\! \equiv\! \widetilde N$. Thus, we also 
get $M\! \equiv\! \widetilde N$.
\end{example}

\vspace*{-1mm}    

Overall, these results enable the use of ABA semantics to reason about
CABA frameworks by translating constrained arguments into standard ABA arguments. However, again, this may require an expensive, possibly infinite, grounding. Attacks between constrained arguments, given next, may avoid this.

\vspace*{-1mm}   

\section{Attacks in CABA}\label{sec:Att}

In this section we introduce  attacks between 
constrained arguments. 
We distinguish between \emph{full} and \emph{partial attacks},
defined in terms of 
different logical relations between the constraints~involved.

\vspace*{-1mm}

\begin{definition}
\label{def:nongroundattacks}
Let $\alpha\!=\!\argu{\{c_1,\ldots,c_m\} \cup A_1}{\sent_1}
$  
and $\beta\!=\!\argu{\{d_1,\ldots,d_n\} \cup A_2\cup \{\asm\}}
{\sent_2}$ be constrained arguments, with $s_1\!=\!\overline{\asm}$.
Then, 
\begin{itemize}
    \item 
$
\alpha$ \emph{fully attacks} $\beta$ if
    $\ctheory \models \forall \, ((d_1\wedge\ldots\wedge d_n) \rightarrow \exists_{-\mathit{vars}(\sent_1)}(c_1\wedge\ldots\wedge c_m))$, 
and 

\item 
$
\alpha$ \emph{partially attacks} $\beta$ if 
    $\ctheory \models \exists\, (\exists_{-\mathit{vars}(\sent_1)}(c_1\wedge\ldots\wedge c_m)\wedge \exists_{-\mathit{vars}(\sent_1)}(d_1\wedge\ldots\wedge d_n)))$.
\end{itemize}
Moreover, let $\alpha'=\argu{\{c_1,\ldots,c_m\} \cup 
A_1}{\overline{p(\tuple{t})}}$  
and $\beta'=\argu{\{d_1,\ldots,d_n\} \cup A_2 \cup \{p(\tuple{u})\}}{\sent_2}
$ 
with $\tuple{t}\!\neq\! \tuple{u}$.
Let $\tuple{X}$ be a tuple of new distinct variables.
Then, $\alpha'$ \emph{fully} (or \emph{partially}) 
\emph{attacks}~$\beta'$ if $\argu{\{c_1,\ldots,c_m, \tuple{X}\!=\!\tuple{t}\} \cup A_1}{\overline{p(\tuple{X})}}
$
    \emph{fully} (or \emph{partially}, resp.) 
    \emph{attacks} $\argu{\{d_1,\ldots,d_n,\tuple{X}\!=\!\tuple{u}\} \cup A_2\cup 
    \{p(\tuple{X})\}}{\sent_2}
    $.

\end{definition}

\vspace*{-2mm}
\begin{example}
\label{ex:existV}
$\argu{\{X \!>\!0,\ Y\!<\!2,\ q(X, Y)\}}{\overline{p(X)}}$  
fully attacks $\argu{\mbox{$\{X\!>\!10,$}$ $Z\!<\!3,\ p(X)\}}{r(Z)}$, 
because
$\exists_{-\{X\}}(X\! >\!0\! \wedge\! Y\!<\!2)$ is  
$\exists Y(X \!>\!0 \wedge $ $\mbox{$Y\!<\!2)$}$ and 
$\ctheory \models \forall X,Z \, \bigl((X\!>\!10 \wedge Z\!<\!3)\! \rightarrow\! 
\exists {Y}(X\!>\!0 \wedge Y\!<\!2)\bigr)$.
\end{example}


In the sequel we use the following property of partial attacks.

\vspace{-1mm}
\begin{proposition}\label{prop:pa}
Let $\alpha,\beta$ be constrained arguments.
Then, $\alpha$ partially attacks $\beta$ iff there exist
$\alpha'=\argu{C \cup A_1}{\sent_1}$ and $\beta'=\argu{D \cup A_2}{\sent_2}$
such that: $(i)$ $\{\alpha\}\!\equiv\!\{\alpha'\}$, $(ii)$ $\{\beta\}\!\equiv\!\{\beta'\}$, $(iii)$ $\sent_1\!=\!\overline{\asm}$, for some $\asm\!\in\! A_2$, $(iv)$ $\mathit{vars}(C)\!\cap\! \mathit{vars}(D) \subseteq \mathit{vars}(\sent_1)$, and
$(v)$ $C\cup D$ is consistent.
\end{proposition}

\begin{example} 
\label{ex:attacks}
Referring to Example~\ref{ex:simple7newMG}, 
$\alpha_5$ fully attacks $\alpha_1$ as 
$\ctheory \models \forall X,Y \, (
(X\!<\!1 \wedge \mbox{$Y\!>\!0$})\!\rightarrow\! X\!<\!10)$.
Instead,
$\alpha_4$ does not fully attack~$\alpha_1$ 
as 
$\ctheory \!\not \models\! \forall X,Y  (
(\mbox{$X\!<\!1$} \! \wedge\!  
\mbox{$Y\!>\!0$})\!\rightarrow\! 
(X\!<\!5 \!\wedge\! Y\!>\!3))$.
However, $\alpha_4$ 
partially attacks $\alpha_2$ 
as 
$\argu{\{X\!<\!5,\ \mbox{$Y\!>\!3$},\ \mbox{$X_1\!=\!X$},\ \mbox{$X_2\!=\!Y$}\}}{\mathit{ca}(X_1,X_2)}$
partially attacks 
$\argu{\mbox{$\{Y\!\!<\!10$}, X_1\!=\!Y, X_2\!=\!X,\ 
a(X_1,X_2)\}}{p(X)}$.  Indeed,
$\ctheory \models \exists X_1,X_2$ \mbox{$(X_1\!\!<\!10$}$ 
\wedge X_1\!<\!5 \wedge X_2\!>\!3)$. 
By Proposition~\ref{prop:pa}, $\alpha_4$ 
partially attacks $\argu{\{X\!\!<\!\!10,a(X,Y)\}}{p(Y)}$ (
a 
renaming of $\alpha_2$).

\end{example}

The following result relates our different notions 
of attack.


\vspace*{-1mm}
\begin{proposition}
\label{prop:attacksrelation}
Let 
$\alpha$  
and $\beta\!\in\! \CArg$.
If 
$\alpha$ fully attacks~$\beta$, then $\alpha$ partially attacks $\beta$.
\end{proposition}

\vspace*{-1mm}

For ABA-as-CABA frameworks, partial and full attacks coincide and are standard attacks in ABA, again pointing to CABA  being a conservative generalisation of ABA. Next, we explore the relation between CABA and ABA attacks for generic CABA frameworks.

\vspace{-1mm}\subsection{Mapping Attacks in CABA to 
ABA}

The following result establishes 
a correspondence between 
full and partial attacks in CABA and  
attacks in standard ABA~\cite{ABAtutorial}.

\vspace*{-1mm}

\begin{theorem}
\label{theo:nongroundattacks}
\!For $\alpha,\!\beta\!\in\!\CArg$, $\Gamma\!=\!\GrCInst(\alpha)$, and
$\Delta\!=\!\GrCInst(\beta)$,
    
\noindent $(i)$
$\alpha$ \emph{fully attacks} $\beta$ iff for each  
    $\beta' \! \in\! \Delta$ there exists 
    $\alpha' \! \in\! \Gamma$ such that 
    $\alpha'$ attacks $\beta'$ in $\mathit{Ground}(F_c)$, and
%
$(ii)$ $\alpha$ \emph{partially attacks} $\beta$  iff  
    there exist 
    $\alpha' \! \in\! \Gamma$ and $\beta' \! \in\! \Delta$ such that 
    $\alpha'$ attacks $\beta'$ in $\mathit{Ground}(F_c)$.

\end{theorem}

\vspace*{-1mm}

Note that, when the notion of attack is considered for ground instances of constrained arguments in CABA frameworks, the distinction between full and partial 
attack collapses. In such a case, we refer to them generically as `attack'.





\section{Extension-based Semantics for CABA}\label{sec:sem}

In this section we explore two ways to equip CABA with a semantics.
First, we leverage on the mapping, by grounding, of 
arguments and  attacks in CABA 
onto corresponding standard ABA notions
(Section~\ref{sec:ground}).
Then, we define novel notions of extensions for CABA, without resorting to up-front mapping to standard ABA
(Section~\ref{sec:non-ground}).
Throughout, we focus on conflict-free, admissible, and stable extension semantics.

\subsection{
CABA Semantics via 
ABA Semantics}
\label{sec:ground}


We can define the extensions of a CABA framework $F_c$ as extensions of the ABA framework $\mathit{Ground}(F_c)$. 

\begin{definition}
\label{def:extbyground}
\!\!A set $\Gamma \!\subseteq\! \CArg$ of constrained arguments is a conflict\--free, or admissible, or stable  extension in $F_c$ if  
$\GrCInst(\Gamma)$ is a conflict-free,\,or
admissible,\,or\,stable\,extension,\,resp.,\,in
$\mathit{Ground}(\!F_c\!)$.
\end{definition}

In order to determine the extensions in $\mathit{Ground}(F_c)$,
we can leverage on 
the mapping to AA, 
as for standard ABA~\cite{ABAtutorial}.

Let $\GroundCArg$ be 
as introduced in 
Definition~\ref{def:consistentgrounding-final} 
(which, by 
Theorem~\ref{theo:GroundCABAvsABA}, is equal  to 
$\mathit{Arg}_c$ from Definition~\ref{def:GroundFc} modulo ground constraints).
Let $\Att\!=\!\{(\alpha,\beta)\!\in \! \GroundCArg \times\GroundCArg \mid \alpha$ attacks $\beta\}$
be the set of all 
attacks in standard ABA  (see Section~\ref{sec:background}).
%
Then,
$\Gamma \subseteq \CArg$
is a  conflict-free (or admissible, or stable)  extension of  $F_c$ iff $\Gamma$ is a conflict-free (or admissible, or stable, resp.) extension of $(\GroundCArg,\Att)$.

\vspace*{-1mm}

\begin{example} 
\label{ex:attackground}
Let $ \GrCInst(M)$ be as in Example~\ref{ex:simplenewgr-final}. 
\label{ex:simple-sem}
The AA framework is $(\!\GrCInst(M),$ $\Att)$, where  
the set $\Att$ 
is as follows, modulo ground constraints, with 
$n,m \!\in \!\mathbb Q$:

\noindent
$\begin{array}{llll}
 \{  \  \bigl(\argu{\{\,\}}{ca(n,m)}, \argu{\{a(n, m), b(n)\}} {p(n)}\bigr)\mid 
   n\!<\!1,\ m\!>\!3\ \} & \cup   \\

    \{\  \bigl(\argu{\{\,\}}{cb(n)}, \ \argu{\{ a(n, m), b(n)\}} {p(n)}\bigr)\mid 
  n\!<\!1,\ m\!>\!0\ \} &\cup \\
   
 \{\  \bigl(\argu{\{\,\}}{ca(n,m)},\ \argu{\{ a(n,m)\}}{p(m)}\bigr) \mid 
    n\!<\!5,\ m\!>\!3\ \}\ & \cup\\ 
\{  \  \bigl(\argu{\{\,\}}{ca(n,m)},  \ \argu{\{a(n, m)\}} {a(n,m)}\bigr)\mid 
   n\!<\!5,\ m\!>\!3\ \} & \cup\\
  \{\  \bigl(\argu{\{\,\}}{cb(n)}, \ \argu{\{ b(n)\}} {b(n)}\bigr)\mid 
  n\!<\!10\ \}.
\end{array}
$


Then, let us consider the following set $\Delta
\subseteq \GrCInst(M)$, modulo 
ground constraints, where $n,m \!\in \!\mathbb Q$: 

\noindent
\makebox[6mm][l]{$\Delta=$}\makebox[43mm][l]{\makebox[23mm][l]{$\{\argu{\{ a(n,m)\}}{p(m)}$} $\mid  n\!\geq\! 5,\  n\!<\!10\}$} $\cup$

\noindent
\hspace{6mm}\makebox[43mm][l]{\makebox[23mm][l]{$\{\argu{\{ a(n,m)\}}{p(m)}$} $\mid
  m\!\leq\!3,\ n\!<\!5\}$} $\cup$
\makebox[28.5mm][l]{\makebox[2mm][l]{\makebox[13mm][l]{$\{\argu{\{\,\}}{s(n)}$} $\mid n\!>\!0\}$}}  $\cup$

\noindent
\hspace{6mm}\makebox[17mm][l]{$\{\argu{\{\,\}}{ca(n,m)}$} 
     $\mid n\!<\!5,\ m\!>\!3\}$ $\cup$
\makebox[36mm][l]{\makebox[14mm][l]{$\{\argu{\{\,\}}{cb(n)}$}  
     $\mid  n\!<\!10\}$} $\cup$

\noindent
\hspace{6mm}\makebox[25mm][l]{$\{\argu{\{a(n,m)\}}{a(n,m)}$} 
     $\mid n\!\geq\!5\}\ \ \cup$
\makebox[36mm][l]{\makebox[25mm][l]{$\{\argu{\{a(n,m)\}}{a(n,m)}$} 
     $\mid m\!\leq\!3\}$} $\cup$ 

\noindent
\hspace{6mm}\makebox[18mm][l]{$\{\argu{\{b(n)\}}{b(n)}$}  
     $\mid n\!\geq \!10\}$.


We have that $\Delta$ is conflict-free, admissible, and the unique stable extension 
of $(\!\GrCInst(M),\Att)$. 
%
\label{ex:stableground} Note that every ground
argument in 
$\{\argu{\{a(n, m),b(n)\}} {p(n)}
     \mid \mbox{$n\!<\!1,$}\,m\!>\!0\}$
     is attacked by an argument in 
     $\{\argu{\{\}}{cb(n)} \mid n\!<\!10\}$ ($\subseteq\!\Delta$) (see Example~\ref{ex:simplenewgr-final}).
     
Now, let us consider the set $\Gamma \subseteq \CArg$ defined as
follows:

\hspace*{-5mm}
$\begin{array}{llll}
    \Gamma\!=\!\bigl\{\ 
  (\alpha_{2,1}) \ \argu{\{Y\!\geq\!5,\  Y\!<\!10,\  a(Y,X)\}\ }{\ p(X)}, &
  \end{array}$

\vspace*{-.5mm}
\hspace*{-5mm}
$ \begin{array}{l@{\hspace{1mm}}lll}
     (\alpha_{2,2,2})\ \argu{\{X\!\leq\!3,\  Y\!<\!5,\  a(Y,X)\}\ }{\ p(X)},
      \hspace*{4mm}(\alpha_{3})\ \argu{\{Y\!>\!0\}}{s(Y)},\\
      (\alpha_{4})\,\ \ \argu{\{X\!<\!5,\ Y\!>\!3\}}{ca(X,Y)}, 
      \hspace*{3.5mm} (\alpha_{5})\ \ \, \argu{\{X\!<\!10\}}{cb(X)},\\
     (\alpha_{6,1}) \ \argu{\{X\!\geq\!5, a(X,Y)\}}{a(X,Y)},\ \  
     (\alpha_{6,2})\  \argu{\{Y\!\leq\!3, a(X,Y)\}}{a(X,Y)},\\
     (\alpha_{7,1})\  \argu{\{X\!\geq\!10,\ b(X)\}}{b(X)} \hspace*{1mm} \bigr\}.
\end{array}
$


The labels $(\alpha_i$)'s relate the associated constrained 
arguments to the elements of $\widetilde{M}$ and 
$\widetilde{N}$, if any (see Examples~\ref{ex:simple7newMG} 
and~\ref{ex:equivalence}).


Trivially, $\GrCInst(\Gamma)=\Delta$ modulo ground constraints.
Then, by Definition~\ref{def:extbyground}, 
$\Gamma$ is a conflict-free, admissible, and 
stable extension of 
the CABA framework $\mathit{FA}$ of Example~\ref{ex:simplenew}.
This extension is the unique stable one, modulo $\equiv$.
\end{example}

\vspace*{-1mm}
Note that, for ABA-as-CABA frameworks 
the notions of conflict-free, admissible, and stable extensions from 
Definition~\ref{def:extbyground} are exactly as in standard ABA, 
allowing us to conclude that CABA  is a conservative generalisation of 
ABA.

\vspace*{-1mm}
\subsection{
Native CABA Semantics}
\label{sec:non-ground}



Now, we 
naturally characterise
conflict-free and stable extensions directly in terms of the notions of full and partial attacks in CABA. 

\vspace*{-1mm}
\begin{definition}\label{def:NGCF}
A set $\Sigma\!\subseteq\! \CArg$ is \emph{non-ground conflict-free $($NGCF$\,)$} 
iff there are no $\alpha, \beta \!\in\! \Sigma$ such that $\alpha$ 
partially attacks $\beta$.  By
$\FAtt(\Sigma)$ we denote the set $\{ \beta\!\in\! \CArg \mid 
\exists \mbox{$\alpha\! \in\! \Sigma$} \text{ such} $ 
$ \text{that } \alpha \text{ fully attacks } \beta \}. $
\end{definition}

\vspace*{-3mm}
\begin{theorem}\label{prop:non-ground-sem}
For any $\Sigma\!\subseteq\! \CArg$ we have that\,$:$

\noindent
$(1)$ $\Sigma$ is conflict-free iff\,
$\Sigma$ is NGCF\,$;$ and 

\noindent
$(2)$ $\Sigma$ is stable iff\, $\Sigma$ is NGCF and
$\Sigma \!\cup\! \FAtt(\Sigma)\! \equiv\! \CArg$.
\end{theorem}

\vspace*{-2mm}

\begin{example}
\label{ex:stable}
For the CABA framework of Example~\ref{ex:simplenew}, we have that the  set $\Gamma$ of arguments (see Example~\ref{ex:attackground}) is conflict-free.
Moreover, $\alpha_1$,
$ \alpha_{2,2,1}$, 
$\alpha_{6,3}\!\!:\! \argu{\{X\!<\!5, Y\!>\!3,
a(X,Y)\}}
{a(X,Y)},$ and
$\alpha_{7,2}\!\!:\! 
\argu{\{\mbox{$X \!<\! 10$},\ b(X)\}}
{b(X)}$  belong to $\FAtt(\Gamma)$. 
Indeed, 
$\alpha_1, \alpha_{2,2,1}, \alpha_{6,3}$, and $\alpha_{7,2}$ are fully attacked by $\alpha_5$, $\alpha_4$, $\alpha_4$, and $\alpha_5$, respectively. 
Since $\{\alpha_2\}\! \equiv\! \{\alpha_{2,1}, \alpha_{2,2,1}, 
\alpha_{2,2,2}\}$,
$\{\alpha_6\}\! \equiv\! \{\alpha_{6,1}, \alpha_{6,2},$ $\alpha_{6,3}\}$, and $\{\alpha_7\}\! \equiv\! \{\alpha_{7,1},$ $\alpha_{7,2}\}$,
we have  that 
$\Gamma \!\cup\! \FAtt(\Gamma) \!\equiv\! \CArg$.
Thus, $\Gamma$ is a stable extension.
Actually, $\Gamma$ is the unique stable extension, modulo $\equiv$. 
\end{example}
\vspace*{-1mm}


 However, the characterisation of 
stable extensions for CABA by Theorem~\ref{prop:non-ground-sem} still refers to the set $\CArg$ of all (ground and non-ground) constrained arguments.
Moreover, it neglects admissible extensions, as these cannot be naturally characterised in terms of  NGCF and $\FAtt$. 
Now, we propose 
a characterisation,
not requiring 
ground instances, and including admissible extensions.

We start off by introducing the following 
notion.


\begin{definition}
\label{def:nongrouncommon}
Two constrained arguments $\alpha$ and $\beta$
are said to 
have \emph{common constrained instances} if {there exist 
$\alpha'\!\in\!\GrCInst(\alpha)$ and $\beta'\!\in\!\GrCInst(\beta)$ such that
$\alpha'\!=\!\beta'$ modulo  ground constraints.}
%
\end{definition}

\vspace*{-2mm}

\begin{example} $\alpha\!=\!\argu{\{X\!\!>\!\!3,\! Y\!\!\!<\!\!0, a(X)\!\}}{p(X)}$ and
$\beta\! =\! \argu{\{\mbox{$Z\!\!>\!\!0$},\mbox{$U\!\!>\!\!1$},\linebreak a(Z), a(U)\!\}}{p(U)}$
have common 
constrained instances. Indeed, 
$\alpha'\!=\!\argu{\{5\!>\!3, \mbox{$-1\!<\!0$},a(5)\!\}}
{p(5)}\!\in\! \GrCInst(\alpha)$,
$\beta'\!=\!\argu{\{5\!>\!0, \mbox{$4\!+\!1\!>\!1$},a(5),\linebreak a(4\!+\!1)\!\}}
{p(4\!+\!1)}\!\in\! \GrCInst(\beta)$,
and $\alpha'\!=\!\beta'$ modulo ground constraints.
\end{example}


The following result allows us to check whether or not constrained arguments have common constrained instances. 

\vspace*{-1mm}

\begin{proposition}
\label{prop:generalizationcommon}
\!\!\!Two\,constrained\,arguments\,$\alpha$\,and $\beta$\,have\,common constrained instances iff there exist 
$\alpha'\!\!=\!\argu{C' \!\cup \!A'\!}{\sent}$
and $\beta'\!\!=\!\argu{D'\! \cup\! A'\!}{\sent}$ such that:
$(i) \{\alpha\}\!\equiv\!\{\alpha'\}$, 
$(ii) \{\beta\}\!\equiv\!\{\beta'\}$,
$(iii) \mathit{vars}(C')\!\cap\! \mathit{vars}(D') \!\subseteq\! \mathit{vars}(\{A',s\})$,
and $(iv)\,C'\!\cup\! D'$ is consistent.
\end{proposition}

\vspace*{-1mm}

Then, we can apply the notion of absence of common constrained 
instances to characterise \emph{instance-disjoint} sets.
We can also 
introduce a way to compare instances of constrained arguments based on 
attacks between them,  
by assessing whether or not
they are \emph{non-overlapping} in that 
every partial attack is a full 
attack.
This notion will help identify 
non-ground instances of most general constrained arguments to choose extensions from. 

\vspace*{-1mm}
\begin{definition}
\label{def:nonoverlapping}
    \!Let $\Delta\!\subseteq \!\CArg$. 
    (i) $\Delta$ is \emph{instance-disjoint} if \mbox{$\forall \alpha,\beta \!\in \!\Delta$,} $\alpha$ and $\beta$ have no common constrained instances; 
    (ii) $\Delta$ is \emph{non-overlapping} if $\forall \alpha,\beta \!\in \! \Delta$,~$\alpha$ partially attacks $\beta$ iff $\alpha$ fully attacks~$\beta$. 
\end{definition}

\begin{theorem}\label{theo:fully-att}
\!\!Let $\Delta\!\subseteq\!\CArg\!$ be\,non-overlapping, and $\Sigma\!\!\subseteq\!\Delta\!\equiv\! \CArg$. 
    
\noindent\hangindent=4mm
$(1)$\ $\Sigma$ is conflict-free iff\ 
$\,\not\exists \alpha, \beta \!\in\! \Sigma$, 
such that $\alpha$ fully attacks $\beta$;

\noindent\hangindent=4mm
$(2)$ $\Sigma$ is admissible iff $\,\Sigma$ is conflict-free and $\,\forall \alpha\! \in\! \Sigma$, if {$\,\exists \beta\!\in\! \Delta$} such that $\beta$ fully attacks $\alpha$, then $\exists \gamma\!\in\! \Sigma$ such that $\gamma$ fully attacks $\beta$;

\noindent\hangindent=4mm
$(3)$\   $(3.1)$
$\Sigma$ is stable if $\,\Sigma$ is conflict-free and {$\,\forall\beta\!\in\!\Delta\!\setminus\! \Sigma$}, $\exists \alpha\!\in\! \Sigma$ such that $\alpha$ fully attacks $\beta$. 

\ $(3.2)$ If $\Delta$ is instance-disjoint, then the converse of $(3.1)$ holds. 
\end{theorem}

\vspace*{-1mm}

Theorem~\ref{theo:fully-att} provides 
native characterisations of CABA semantics in 
terms of full attacks.
Points (1) and (3) 
give a native characterisation of conflict-free 
and stable extensions not requiring 
ground
instances (unlike 
Theorem~\ref{prop:non-ground-sem}). 
Point~(2) gives a native characterisation of 
admissible extensions.
The following example shows the need of the 
instance-disjoint condition on $\Delta$ in
Point\,(3.2). 

\vspace*{-1mm}
\begin{example}\label{ex:non-disjoint}
Let $\Delta = \bigl\{ \alpha_1\mbox{:} \argu{\{X \!>\! 0\}}{p(X)}, \alpha_2\mbox{:}\argu{\{X\! >\! 3\}}{p(X)} \bigr\} \equiv \CArg.$   
$\Delta$ is non-overlapping, but not instance-disjoint.  
By Definition~\ref{def:extbyground},  
$\Sigma=\bigl\{ \alpha_1 \bigr\}$
is stable.
However, $\alpha_1$ does not attack $\alpha_2$.
\end{example}

\vspace*{-1mm}

Finally, we present a procedure, called 
\emph{Argument Splitting}, that, given a set 
of constrained arguments,  
constructs an equivalent instance-disjoint, 
non-overlapping set. The procedure requires the constraint theory $\ctheory$ to satisfies 
the conditions we now specify.

\begin{definition}
$\ctheory$ is \emph{closed under negation} if, for every 
$c\in\constr$, there exist 
$d_1,\ldots,d_m\in \constr$ such that $\ctheory\models \forall (\neg c \leftrightarrow (d_1\vee\ldots\vee d_m))$, and, for $i,j=1,\ldots,m$ with $i\!\neq\! j$, $\ctheory\models\neg\exists (d_i\wedge d_j)$. 
$\ctheory$ is \emph{closed under existential quantification} if, for every  $\{c_1,\ldots,c_m\} \subseteq \constr$ 
and 
set $V$ of variables, there exist 
$d_1,\ldots,d_i, \ldots, d_j,\ldots, d_{n} \in \constr$
such that $\ctheory \models \forall (\exists_{-V} (c_1\wedge \ldots \wedge c_m) \leftrightarrow ((d_1\wedge \ldots\wedge d_i) \vee  \ldots \vee (d_j\wedge\ldots \wedge d_n)))$.     
\end{definition}


\def\splitci{\mathit{split}_{\mathit{ci}}}
\def\splitpa{\mathit{split}_{\mathit{pa}}}

When $\ctheory$ is closed under negation and existential 
quantification, we can define two splitting operations on pairs 
$(\alpha,\beta)$ of constrained arguments:
(1)~$\splitci(\alpha,\beta)$, which is applied if $\alpha$ and $\beta$ have common constrained instances, and
(2)~$\splitpa(\alpha,\beta)$, which is applied to $\alpha$ and $\beta$ such that $\alpha$ partially attacks and 
does not fully attack $\beta$.
Those operations construct suitable new sets of arguments which are equivalent to $\{\beta\}$ (see Proposition~\ref{prop:split} below).
Now we need the following  notion.

\begin{definition}\label{def:cs}
Let  $C=\{c_1,\ldots,c_m\}$ and $D=\{d_1,\ldots,d_n\}$ be sets of constraints. 
By \emph{constraint split} of $C$ and $D$, we denote a set $\mathit{cs}(C,D) =\{E_1,\ldots,E_r\}$ of sets of constraints where, for $i=1,\ldots,r,$ $e_i$ is the conjunction of the constraints in $E_i$, such that:
$(1)~\ctheory\models\forall ((\neg \exists_{-V}(c_1\wedge\ldots\wedge c_m) \wedge d_1\wedge \ldots \wedge d_n) \leftrightarrow (e_1 \vee \ldots \vee e_r))$, with $V\!=\!\mathit{vars}(C)\cap \mathit{vars}(D)$; 
$(2)~\ctheory\models\exists (e_i)$, i.e., $e_i$ is consistent;
$(3)$~for $j\!=\!1,\ldots,\!r$ with $i\!\neq\! j$, $\ctheory\models\neg\exists (e_i\wedge e_j)$, i.e., $e_i$ and $e_j$ are \emph{mutually exclusive}.
\end{definition}
\vspace*{-1mm}
The next proposition shows that, if $\ctheory$ is closed under negation and existential quantification, constraint split is indeed defined for all pairs  of finite sets of constraints.

\begin{proposition}\label{prop:cs}
If $\ctheory$ is closed under negation and existential quantification, then for all finite sets $C, D\! \subseteq\!\constr$ such that $C\cup D$ is consistent, there exist $E_1,\ldots,E_r$ such that $\mathit{cs}(C,D) = \{E_1,\ldots,E_r\}$.  
\end{proposition}

Now, based on constraint split, we define 
$\splitci$ and $\splitpa$.
\vspace*{-1mm}
\begin{definition}
\label{def:splitci}
Let $\{\alpha\}\equiv\{\argu{C \cup A}{\sent}\}$
and $\{\beta\}\equiv\{\argu{D \cup A}{\sent}\}$ such that: 
(i) $\mathit{vars}(C)\cap \mathit{vars}(D) \subseteq \mathit{vars}(\{A,s\})$,
and (ii) $C\cup D$ is consistent. 
Assume that $\alpha$ and $\beta$ have common constrained instances.
Then,   $\splitci(\alpha,\beta)= \{\beta_1,\ldots,\beta_r\}$, where $\mathit{cs}(C,D) = \{E_1,\ldots,E_r\}$ and, for $i=1,\ldots,\!r,$ $\beta_i=\argu{E_i\cup A}{\sent}$.
\end{definition}

Due to Proposition~\ref{prop:generalizationcommon}, in $\splitci(\alpha,\beta)$ it is not restrictive to require that $\{\alpha\}\equiv\{\argu{C \cup A}{\sent}\}$ and $\{\beta\}\equiv\{\argu{D \cup A}{\sent}\}$.
$\splitci$ is well-defined, as, for $i\!=\!1,\ldots,r,$ $E_i$ is a consistent set of constraints.

\vspace*{-1mm}
\begin{definition}
\label{def:splitpa}
Let $\{\alpha\}\!\equiv\!\{\argu{C\! \cup\! A_1}{\sent_1}\}$ and $\{\beta\}\!\equiv\!   \{\argu{D\! \cup\! A_2}{\sent_2}\}$
such that: (i) $\sent_1\!=\!\overline{\asm}$, for some $\asm\!\in\! A_2$, (ii) $\mathit{vars}(C)\!\cap\! \mathit{vars}(D)\! \subseteq\! \mathit{vars}(\sent_1)$, and
(iii) $C\!\cup\! D$ is consistent.
Assume that $\alpha$ partially attacks $\beta$.
Then, $\splitpa(\alpha,\beta)\!=\! \{\beta_0,\beta_1,\ldots,\!\beta_r\}$, where (1)\,$E_0 = C\cup D$, (2)\,$\mathit{cs}(C,D) = \{E_1,\ldots,E_r\}$ and, (3)\,for $i\!=\!0,\ldots,\!r$, $\beta_i=\argu{E_i\!\cup\! A_2}{\sent_2}$.
\end{definition}

Similarly to the case of $\splitci$, due to Proposition~\ref{prop:pa}, in the above definition it is not restrictive to assume that $\{\alpha\}\equiv  \{\argu{C \cup A_1}{\sent_1}\}$ and $\{\beta\}\equiv\{\argu{D \cup A_2}{\sent_2}\}$.
$\splitpa$ is well-defined, as for $i=0,\ldots,r,$ $E_i$ is a consistent set of constraints.
It can also be seen that the outputs of the two splitting operations are defined up to equivalence. That is, if in Definition~\ref{def:splitci} we consider $\{\alpha\}\equiv\{\argu{C' \cup A'}{\sent'}\}$ and $\{\beta\}\equiv\{\argu{D' \cup A'}{\sent'}\}$ such that Conditions (i)-(ii) hold for $C',A',s',D'$ instead of $C,A,s,D,$ then $\splitci(\alpha,\beta)= \{\beta'_1,\ldots,\beta'_r\}$ with $\{\beta_1\}\equiv\{\beta'_1\},\ldots,\{\beta_r\}\equiv\{\beta'_r\}$. A similar property holds for $\mathit{split}_{\mathit{pa}}$.

Now we show that \red{the two operations} $\mathit{split}_{\mathit{ci}}$ and  $\mathit{split}_{\mathit{pa}}$
preserve equivalence.

\vspace*{-1mm}
\begin{proposition}
\label{prop:split}
Let $\ctheory$ be closed under negation and existential quantification and let $\alpha,\beta\in \Delta\subseteq \CArg$. 

\noindent\hangindent=4mm
$(1)$ 
Assume that $\alpha$ and $\beta$ have common constrained instances and $\splitci(\alpha,\beta)\!=\!\{\beta_1,\ldots,\beta_r\}$.
Then, 
$(i)$~there is no pair in $\{\alpha,\beta_1,$ $\ldots,\beta_r\}$ with common constrained instances, and 
$(ii)$ $\mbox{$(\Delta\!\setminus\!\{\beta\})$} \cup\textit{split}_{ci}(\alpha,\beta) \!\equiv\! \Delta$. 

\noindent\hangindent=4mm
$(2)$ Assume that $\alpha$ partially attacks $\beta$ and $\splitpa(\alpha,\beta)\!=\!\{\beta_0,\beta_1,\ldots,$ $\beta_r\}$.
Then, $(i)~\alpha$ fully attacks $\beta_0$, $(ii)~\alpha$ does not partially attack any argument in $\{\beta_1, \ldots, \beta_r\}$, and $(iii)~(\Delta\!\setminus\!\{\beta\})\cup\textit{split}_{pa}(\alpha,\beta)\equiv\Delta$.
\end{proposition}

\vspace*{-2mm} 

\begin{example}[Ex.~\ref{ex:non-disjoint} contd.]
$\Delta$ is transformed into the equivalent, instance-disjoint set $\Delta\setminus\{\alpha_2\}$, as $\splitci(\alpha_2,\alpha_1)=\{\}$.
\end{example}

\vspace*{-2mm} 

\begin{example}\label{ex:non-overlapping}
Let us consider the CABA framework with rules:

\vspace{1mm}
\makebox[45mm][l]{$\mathit{cp}(X) \leftarrow r(X),\ q(X)$}
$\mathit{cq}(X) \leftarrow s(X),\ p(X)$

\makebox[45mm][l]{$r(X) \leftarrow X\!\geq\! Y, Y\!\geq\! 0$}
$s(X) \leftarrow X\!\leq\! 0$

\vspace*{1mm}
\noindent
where 
$q(X), p(X)$ are assumptions with contraries $cq(X), cp(X)$, respectively. 
The most general arguments are:

\vspace*{1mm}

\makebox[45mm][l]{$\alpha_1$: $\argu{\{X\!\geq\!Y,\ Y\!\geq\! 0,\ q(X)\}}{\mathit{cp}(X)}$}
$\alpha_2$: $\argu{\{X\!\leq\!0,\ p(X)\}}{\mathit{cq}(X)}$

\makebox[45mm][l]{$\alpha_3$: $\argu{\{ X\!\geq\! Y,\ Y\!\geq\! 0\}}{r(X)}$}
$\alpha_4$: $\argu{\{X\!\leq \!0\}}{s(X)}$

\makebox[45mm][l]{{$\alpha_5$: $\argu{\{p(X)\}}{p(X)}$}}
{$\alpha_6$: $\argu{\{q(X)\}}{q(X)}$}.

\vspace*{1mm}
\noindent This set of arguments is instance-disjoint, but not non-overlapping: $\alpha_1$ and $\alpha_2$ partially attack each other,
and they also partially attack $\alpha_5$ and $\alpha_6$, respectively. 
Now, 
$\splitpa(\alpha_1,\alpha_2)$ gets:

\vspace*{1mm}
\makebox[45mm][l]{$\alpha_{2,1}$: $\argu{\{X\!=\!0,\ p(X)\}}{\mathit{cq}(X)}$}
$\alpha_{2,2}$: $\argu{\{X\!<\!0,\ p(X)\}}{\mathit{cq}(X)}$.

\vspace*{1mm}
\noindent
The constraint of $\alpha_{2,1}$ is obtained: (i)~by adding the constraints of $\alpha_1$ (i.e. 
$X\!\geq\! Y, Y\!\geq\! 0$) to the constraint of $\alpha_2$ (i.e. $X\!\leq\! 0$), thereby 
getting $X\!\geq\! Y, Y\!\geq\! 0, X\!\leq\! 0$, and then, (ii)~by eliminating 
$Y$, which occurs neither in the claim nor in the assumption of 
$\alpha_2$, thereby getting $X\!=\!0$ by Proposition~\ref{prop:equivproperties} (3).
The constraint of $\alpha_{2,2}$ is obtained by adding the formula 
$\neg \exists Y (X\!\geq\! Y \!\wedge\! Y\!\geq\! 0)$ to the constraint of $\alpha_2$, 
thereby getting, modulo equivalence in $\ctheory$, the constraint $X\!<\!0$.
Similarly, by $\splitpa$ and 
Proposition~\ref{prop:equivproperties},
we replace $\alpha_1, \alpha_5, \alpha_6$ by:

\vspace*{1mm}

\makebox[45mm][l]{$\alpha_{1,1}$: $\argu{\{X\!=\!0,\ q(X)\}}{\mathit{cp}(X)}$}
$\alpha_{1,2}$: $\argu{\{X\!>\!0,\ q(X)\}}{\mathit{cp}(X)}$

\makebox[45mm][l]{$\alpha_{5,1}$: $\argu{\{X<0, p(X)\}}{p(X)}$}
$\alpha_{6,1}$: $\argu{\{X< 0, q(X)\}}{q(X)}$

\makebox[45mm][l]{$\alpha_{5,2}$: $\argu{\{X=0, p(X)\}}{p(X)}$}
$\alpha_{6,2}$: $\argu{\{X= 0, q(X)\}}{q(X)}$

\makebox[45mm][l]{$\alpha_{5,3}.$ $\argu{\{X>0, p(X)\}}{p(X)}$}
$\alpha_{6,3}$: $\argu{\{X> 0, q(X)\}}{q(X)}$.

\vspace*{1mm}
\noindent
\blue{Then,} the extension $\{\alpha_{1,1}, \alpha_{1,2}, \alpha_{2,1}, \alpha_{2,2}, \alpha_3, \alpha_4, \alpha_{5,1}, \alpha_{5,2},$ 
$\alpha_{5,3},$  $\alpha_{6,1},$ $\alpha_{6,2}$, $\alpha_{6,3}\}$
is instance-disjoint and non-overlapping.
\end{example}

The following \emph{Argument Splitting} procedure 
repeatedly applies the two splitting operations to 
get an instance-disjoint, non-over\-lapping set of
constrained arguments.

\vspace*{2mm}

\begin{tabular}{@{\hspace{-3mm}}ll}
\textsc{Procedure} \textit{Argument Splitting.} 
\quad
\emph{Input}: 
$\Delta\!\subseteq\!\CArg$.\nopagebreak\\

\makebox[10mm][l]{\textsc{repeat}}\\

\makebox[3.5mm][l]{}\textsc{if} $\exists \alpha, \beta\!\in\!\Delta$ with common constrained instances\\

\makebox[10mm][l]{}\textsc{then}
$\Delta := (\Delta\!\setminus\!\{\beta\}) \cup \splitci(\alpha,\beta)$;\\

\makebox[3.5mm][l]{}\textsc{if} $\exists\alpha, \beta\!\in\!\Delta$ s.t. $\alpha$ partially attacks and does not 
fully attack~$\beta$\\

\makebox[10mm][l]{}\textsc{then}
$\Delta := (\Delta\!\setminus\!\{\beta\}) \cup \splitpa(\alpha,\beta)$;\\[-.5mm]

\textsc{until} $\Delta$ is instance-disjoint and non-overlapping.
\end{tabular}

\vspace*{2mm}

From Proposition~\ref{prop:split}, we immediately
get the following result.

\vspace*{-1mm}
\begin{theorem}\label{thm:OnArgumSplitting}
If the Argument Splitting procedure terminates for input $\Delta$ and its output is $\Delta'$, then
(i) $\Delta\equiv\Delta'$ and (ii) $\Delta'$ is instance-disjoint and non-overlapping.
\end{theorem}

\vspace*{-1mm}
Now, by combining Theorem~\ref{theo:fully-att} and 
Argument Splitting,
we get a computational method to construct admissible or stable extensions, if the given CABA framework admits any.
First, we construct $\MGCArg$, which, by Corollary~\ref{cor:groundArg=CGround(M)}, is equivalent to $\CArg$. Then, by Argument Splitting, we produce $\Delta'\subseteq\CArg$ which is a non-overlapping, instance-disjoint set of constrained arguments equivalent to $\MGCArg$. 
Finally, we can identify admissible and stable extensions by a direct application of Theorem~\ref{theo:fully-att} on $\Delta'$.

However, in order to construct admissible and stable extensions, $\Delta'$ should be a \emph{finite} set of 
constrained arguments. In general, the problem of
knowing whether or not a given CABA framework has a finite $\Delta'$ is undecidable. 
We leave it for future work to characterize classes of CABA frameworks for which it is possible to construct finite, instance-disjoint, non-overlapping sets of constrained arguments equivalent to $\MGCArg$.

\vspace*{1mm}
The following example shows how admissible and stable extensions can be computed using the output of {Argument Splitting}.

\vspace*{-1mm}
\begin{example}
[Ex.~\ref{ex:non-overlapping} contd.]\label{ex:non-g-ext}
The CABA framework has the following two stable, and hence admissible, extensions: 
$E_1\!=\!\{\alpha_{1,1}, \alpha_{1,2},$ 
$ \alpha_{2,2}, \alpha_3, \alpha_4,$  $\alpha_{5,1},$ $\alpha_{6,2}, \alpha_{6,3}\}$ 
and 
$E_2\!=\!\{\alpha_{1,2}, \alpha_{2,1}, \alpha_{2,2}, \alpha_3, \alpha_4, \alpha_{5,1}, \alpha_{5,2},$ $\alpha_{6,3}\}$. 
We  have that:
$\alpha_{1,1}\!\in\! E_1$ attacks $\alpha_{2,1}\!\not\in\! E_1$, and $\alpha_{2,1}\!\in\! E_2$ attacks $\alpha_{1,1}\!\not\in\! E_2$.
Additional admissible, non-stable extensions are, among others: $\{\alpha_{1,1}, \alpha_3, \alpha_4\}$ and $\{\alpha_{2,1}, \alpha_3, \alpha_4\}$. 
\end{example}
\vspace*{-1mm}

Thus, by our method we can effectively construct finite non-ground CABA (admissible or stable) extensions that are equivalent to ground ABA extensions consisting of infinite sets.

\section{Conclusions}

We define CABA frameworks,  allowing us to use argumentative reasoning in cases 
that require constraints over infinite 
domains.
The semantics for CABA is defined in terms of
constrained arguments and attacks between them. 
Each constrained argument represents a possibly infinite set of ground instances obtained by substituting the variables therein.
This semantics does not require an expensive, potentially infinite, grounding.

This work opens many avenues for future research, including the following ones.

We have focused on conflict-free, admissible, and stable extensions only, but it would be interesting to study other extension-based semantics, 
e.g. preferred extensions,
complete extensions, and grounded extensions.
We have focused on the flat case, but it would be interesting to define non-flat CABA, 
when assumptions may occur in the head of rules. 
Similarly, it would be interesting to consider variants of CABA frameworks,
in the spirit of variants of ABA frameworks, e.g. with preferences between assumptions~\cite{ABAhandbook}, or probability distributions over assumptions~\cite{DBLP:conf/comma/DungT10}.

We have defined CABA frameworks as ABA frameworks with additional components and notions.  It would be interesting to explore whether our CABA  could be obtained as an instance of standard ABA, for an appropriate choice of deductive system, assumptions and contraries, e.g. in the spirit of default logic~\cite{ABA}.
Given that, for the definition of CABA, we drew inspiration from CLP, 
one could study relations between various forms thereof and CABA.

CABA addresses a very specific computational issue with deploying ABA  without expensive, potentially infinite, grounding.
It would be interesting to study its computational complexity as well as providing 
computational machinery therefor, either via mapping to ASP-based CLP systems~\cite{s(CASP)} for stable extensions, or via dispute derivations  
\cite{ABAhandbook} for admissible extensions.
It would also be useful to consider classes of constraint domains 
where the reasoning tasks are decidable.

Finally, 
it would be interesting to explore 
applications of CABA in 
settings where constraints naturally emerge, 
as in the case of legal 
reasoning, in the spirit of the motivating Example~\ref{ex:motive}.

\begin{acks}
Toni was  funded by
the ERC (ADIX, grant no.101020934).  
De Angelis and Proietti were supported by MUR PRIN 2022 Project DOMAIN funded by the EU–NextGenEU, M4.C2.1.1, CUP B53D23013220006.
De Angelis, Fioravanti, Proietti, and Meo acknowledge the support of PNRR MUR project PE0000013-FAIR.
De Angelis, Fioravanti, Meo, Pettorossi, and Proietti are members of INdAM-GNCS. 
\end{acks}

\bibliographystyle{ACM-Reference-Format}
\balance
\bibliography{sample}

\ifthenelse{\equal{\version}{extended}}
{
\newpage

\appendix 

\newcommand{\eop}{\hfill$\Box$}

\section*{Appendix}




\balance

\section{CABA Formalization of the Introductory Example}
In the Introduction we have presented the motivating 
Example~\ref{ex:motive}. Now we show that it can be 
that it 
can be formalized as a CABA framework $\langle \langc, \constr, \RR, \ctheory$, $\A, \contrary\rangle$. 

First, we introduce the following sets.
Let $\mathbb P$ be a set of constants (that is, function symbols with arity~0), denoting persons, and $\mathbb Q$ be the set of rationals.
Let $V_{\mathbb{P}}$ and $V_\mathbb{Q}$ be two disjoint 
sets of variables associated with $\mathbb P$ and $\mathbb Q$, respectively.

Let $T_{\mathbb P}$ be the set of terms 
built out  of the elements of $\mathbb P \cup V_{\mathbb P}$,
and~$T_{\mathbb Q}$ be the set of terms 
built out of the elements of $\mathbb Q \cup V_{\mathbb Q}$ and the 
binary function symbols $+$ and $-$, denoting the usual \emph{plus} 
and \emph{minus} operations on rationals.  
Thus, by construction, we may associate the sort \emph{person} with every term in $T_{\mathbb P}$ and the sort \emph{rational} with every term in $T_{\mathbb Q}$.

Let us also consider the following predicates:
(i)~$\textit{must\_pay\_tax}$, 
$\textit{nonex}\-\textit{empt}$, $\textit{exempt}$, 
$\textit{salary\_income}$, and $\textit{other\_incomes}$, 
whose arguments are terms in $T_{\mathbb P}$, and (ii)~$\textit{income}$ and
$\mathit{foreign\_}\mathit{income}$, whose arguments are
pairs of terms in $T_{\mathbb P} \times T_{\mathbb Q}$.

Then, we have:

\noindent \hangindent=8mm
$\ \bullet\ \langc\!=\!\{\textit{must\_pay\_tax}(p), 
\textit{nonexempt}(p), \textit{exempt}(p),$\\ 
\hspace*{3mm}$\textit{salary\_income}(p), \textit{other\_incomes}(p),$
$\textit{income}(p,i),$ \\
\hspace*{3mm}$\textit{foreign\_income}(p,i) \mid p\! \in\! T_{\mathbb P},\  i\!\in\! T_{\mathbb Q} \}\cup \constr$;

\noindent 
$\ \bullet\ \constr = \{t_1\!<\! t_2$, $t_1\!\leq\!t_2$,\ $t_1\!=\!t_2,\ 
t_1\!\not =\!t_2,\ t_1\!\geq\! t_2,\ t_1\!>\! t_2 \mid t_1, t_2\in 
T_{\mathbb Q}\}$;

\noindent 
$\ \bullet\ \RR=\{R1, R2, R3\}$,  where:

$R1.\ \textit{must\_pay\_tax}(P) \!\If\! \textit{income}(P,I),\ 
I\!\geq\!0,\ \textit{nonexempt}(P)$

$R2.\ \textit{exempt}(P)\! \If\! \textit{income}(P,I), I\!\geq\!0, 
I\!\leq\! 16000,  \textit{salary\_income}(P)$

$R3.\ \textit{other\_incomes}(P) \!\If\! \textit{foreign\_income}
(P,F),\ F\!\geq\! 10000$

with $P\!\in\! V_{\mathbb P}$ and $I\!\in\! V_{\mathbb Q}$.

\noindent 
$\ \bullet\ \ctheory$ is 
the theory of linear rational arithmetic (LRA);

\noindent 
$\ \bullet\ \A=\{\textit{nonexempt}(p), \textit{salary\_income}(p) \mid p\!\in\! T_{\mathbb P}\}$;

\noindent 
$\ \bullet$ for all $p \in T_{\mathbb P}$, 
$\overline{\textit{nonexempt}(p)}=\textit{exempt}(p)$  and \\
\hspace*{21mm}$\overline{\textit{salary\_income}(p)}=
\textit{other\_incomes}(p)$.

In this CABA framework we have considered the operations $+$ and $-$, and the 
constraint
relations $<, =, \neq$, and $>$, although they do not occur in 
Example~\ref{ex:motive}. We did so simply for the objective of showing the flexibility and 
the power of the CABA formalism. In particular, if one had to consider variants of that
example which have more rules, those extra rules
can easily be 
included into our CABA framework by
using the already available operations and constraint relations.

\section{Proofs}\label{app:proofs}

In the proofs, we use the following notation
and terminology.

Given a substitution $\vartheta = \{X_{1}/t_{1},\ldots, X_{n}/t_{n}\}$, the set $\{X_{1},\ldots, X_{n}\}$ is the \emph{domain} of $\vartheta$, denoted by $\mathit{dom}(\vartheta)$.  The set $\{\vartheta(X)\,|$ $X\!\in\!\mathit{dom}(\vartheta)\}$ is the \emph{range}  of~$\vartheta$, denoted $\mathit{rng}(\vartheta)$.  
A pair $\{X_{i}/t_{i}\}$, for $1\!\leq\!i\!\leq\!n$, in~$\vartheta$ is said to be a \emph{binding} for $X_i$.
If $\vartheta$ is a substitution and $\tuple{t}$ is a tuple of terms, we denote by $\vartheta|_{\tuple{t}}$
the restriction of $\vartheta$ to
$\mathit{vars}(\tuple{t})$, that is, 
$\mathit{dom}(\vartheta|_{\tuple{t}})\!=\!\mathit{vars}(\tuple{t})$ and for all $X\!\in\!\mathit{vars}(\tuple{t})$, $(\vartheta|_{\tuple{t}})(X)\!=\!\vartheta(X)$. Similarly, given any syntactic expression $e$, $\vartheta|_e$ denotes $\vartheta$ restricted to $\mathit{vars}(e)$.
The substitution with empty domain is the \emph{empty substitution} $\{\}$ (also called the 
\emph{identity substitution}), denoted $\varepsilon$.
Given a formula $\varphi$, we denote by 
$\mathit{FV}(\varphi)$ the set of the free variables occurring in $\varphi$, i.e., the variables in $\varphi$ that are outside the scope of a quantifier.


The composition $\vartheta\sigma$ of the substitutions 
$\vartheta$ and $\sigma$ is defined as the functional composition when considering
a substitution as a finite function (for technical details see~\cite{Lloyd87}). Since every renaming is a finite permutation,
for every renaming $\rho\!=\! \{X_{1}/Y_{1},\ldots, X_{n}/Y_{n}\}$, there exists the inverse renaming $\rho^{-1}= \{Y_{1}/X_{1},\ldots, Y_{n}/X_{n}\}$ such that $\rho \rho^{-1}$ = $\rho^{-1}\rho= \varepsilon$.

\subsection*{Proof of Theorem~\ref{thm:grounding}}
The result follows directly from the definitions. \eop

\subsection*{Proof of Proposition~\ref{prop:equalmostconstr}}
 In the following, given an argument $\alpha$, a \emph{partial tree} for 
 $\alpha$ is a tree obtained from the tree $\alpha$ by replacing 
$k$ $(\geq\!0)$ of its disjoint subtrees by their root nodes. Thus, 
we may say that 
a partial tree for $\alpha$ is an \emph{upper portion} of $\alpha$, 
assuming that the children of a node are placed below their father node.
 We represent the subtree consisting of the single root node~$s$  of 
 $\alpha$ by $\argu{\emptyset}{\sent}$. 
 
 Now, let us separately prove the
\emph{Only-if part} and the \emph{If part} 
 of Proposition~\ref{prop:equalmostconstr}.

(\emph{Only-if part}) Assume that $\alpha=\argur{C\cup A}{p(\tuple{t})}{\ruleset} \in \TCArg$. We have to prove that for each partial tree 
$T_\alpha$ of $\alpha$ there exists a partial tree  $T_\beta$ of $\beta=\argur{C'\cup A'}{p(\tuple{X} )}{\ruleset}\in \MGCArg$ such that for $\vartheta=\{\tuple{X} /\tuple{t} \}$, we have $T_\alpha=T_\beta \vartheta$, that is, 
 $C\!=\!C'\vartheta$, $A\!=\!A'\vartheta$, and $C'\vartheta$ is consistent.
         
        
The proof proceeds by complete induction on the number of nodes $n$ of the partial tree $T_\alpha$.

\noindent
 \emph{Basis}. ($n\!=\!1$). In this case the partial tree             
$T_\alpha = \ \argu{\emptyset}{p(\tuple{t})}$ contains only the root node labeled $p(\tuple{t})$.
 $T_\beta  = \ \argu{\emptyset}{p(\tuple{X})}$ is a partial tree for $\beta$, and obviously,
$T_\alpha=T_\beta \vartheta$, for 
$\vartheta=\{\tuple{X}/\tuple{t}\}$.

\noindent
 \emph{Step}. ($n\!>\!1$). Assume the thesis holds for all partial 
 trees with fewer than $n$ nodes. Consider a partial tree  
 $T_\alpha=\argur{C_1 \cup S_1}{p(\tuple{t})}{\ruleset_1}$ with 
 $n\,(\geq \!1)$ nodes, where $C_1$ is a consistent set of constraints 
 in~$\constr$ and $S_1\subseteq \langc\setminus \constr$. By 
 construction of $T_{\alpha}$ there exist: (i)~a partial tree 
 $T'_\alpha=\argur{C_2 \cup S_2}{p(\tuple{t})}{\ruleset_2}$,
and (ii)~a renamed apart rule $r=q(\tuple{V}) \If C_r, \sent_1, \ldots, \sent_m$, such that $S_2=S_3 \cup \{q(\tuple{u})\}$,

\vspace{1mm}
$\sigma =\{\tuple{V}/\tuple{u}\}, \quad
     C_1 =C_2 \cup C_r\sigma, \quad S_1=S_3 \cup \{\sent_1, \ldots, \sent_m\}\sigma$,   \hfill(1)

\vspace{1mm}



\noindent
and $\ruleset_1=\ruleset_2 \cup\{r\}$.
By inductive hypothesis, there exists an partial tree $T'_\beta =
  \argur{C'_2 \cup S'_2}{p(\tuple{X})}{\ruleset_2}$  associated with 
  $T'_\alpha$ such that
         $T'_\alpha= T'_\beta \vartheta$. Thus, we have that:

\vspace{1mm}
     
   $(C'_2 \cup S'_2) \vartheta= C_2 \cup S_2.$ \hfill (2)
   
   \vspace{1mm}
   

\noindent
Therefore, there exists  $q(\tuple{u'}) \in S'_2$  such that: 
     
    $ \begin{array}{lll}
      S_3'= S'_2 \setminus \{q(\tuple{u'})\}, \quad
     S_3'\vartheta=S_3,$ \  and \ 
           $q(\tuple{u'}) \vartheta=q(\tuple{u}).
     \label{eq:s'vartheta}
    \end{array} $ \hfill (3)

\vspace{1mm}

\noindent
Then, by the definition of an argument, we can construct the partial tree 
        $T_\beta=\argur{C_1' \cup S'_1}{p(\tuple{X})}{\ruleset_1}$, where:

\vspace{1mm}
 $\begin{array}{l@{\hspace*{27mm}}l}  \sigma' =\{\tuple{V}/\tuple{u'}\}, \quad
        C'_1 =C_2' \cup C_r\sigma', \mbox{\quad and } \label{eq:sigma1'} & (4)\\  

        S'_1=S'_3 \cup \{\sent_1, \ldots, \sent_m\}\sigma'.  
        \label{eq:sigma'} & (5)
 \end{array}$
    
\vspace{1mm}
\noindent
In order to conclude the proof \emph{Only-if part}, we have to show that:

\vspace{1mm}
$\ (\argur{C_1' \cup S'_1}{p(\tuple{X})}{\ruleset_1} )\vartheta= \argur{C_1 \cup S_1}{p(\tuple{t})}{\ruleset_1},$
\vspace{1mm}

\noindent
that is, (i)~$C_1'\vartheta = C_1$,\ (ii)~$S_1'\vartheta = S_1$,\  and \ 
 (iii)~$p(\tuple{X})\vartheta=p(\tuple{t})$.

Now, $p(\tuple{X})\vartheta=p(\tuple{t})$ holds by construction.
Moreover, since the clauses are renamed apart, 
 $\mathit{vars}(C_r \cup \{\sent_1, \ldots, \sent_m\}) \cap \mathit{vars}(\tuple{X})= \emptyset$, and thus,   
 
 $\begin{array}{lll}
C_r \vartheta=C_r \mbox{\quad and\quad }
\{\sent_1, \ldots, \sent_m\} \vartheta= \{\sent_1, \ldots, \sent_m\}.
\end{array}$ \hfill $(6)$

\noindent
We also have that:

\makebox[9mm][l]{$C_r \sigma' \vartheta$}$= \mbox{\{by (6) and  (4)\}} =
     C_r \{\tuple{V}/\tuple{u'\vartheta}\} = \mbox{\{by (3)\}} =$
     
\makebox[9mm][l]{}$= C_r \{\tuple{V}/\tuple{u}\} = \mbox{\{by (1)\}} =
C_r \sigma.$    \hfill (7)

\vspace{1mm}

\noindent
Analogously,

$ \{\sent_1, \ldots, \sent_m\}\sigma' \vartheta=\{\sent_1, \ldots, \sent_m\} \sigma. $\hfill (8)


\vspace{1mm}

\noindent
Therefore, \nopagebreak

\vspace{1mm}
\makebox[7mm][l]{$ C'_1  \vartheta$}$= \mbox{\{by (4)\}} =
  (C_2' \cup C_r\sigma') \vartheta = \mbox{\{by (2)\}} = $
  
\makebox[7mm][l]{}  $=C_2 \cup C_r\sigma'\vartheta = \mbox{\{by (7)\}} =
  C_2 \cup C_r\sigma = \mbox{\{by (1)\}} =    C_1.$

\vspace{1mm}

\noindent
Moreover,   
\vspace{1mm}

\makebox[7mm][l]{$S'_1 \vartheta$}$= \mbox{\{by (5\}} =
  (S'_3 \cup \{\sent_1, \ldots, \sent_m\}\sigma')\vartheta = \mbox{\{by (8) and  (3)\}} =$
  
 \makebox[7mm][l]{}$=S_3 \cup \{\sent_1, \ldots, \sent_m\}\sigma = \mbox{\{by (1)\}} = 
 S_1.$

\vspace{1mm}

\noindent
This concludes the proof of the \emph{Only-if part}. 

\vspace{1mm}
       
(\emph{If part}). The proof is analogous to that of the \emph{Only-if part} and is omitted. \eop

\subsection*{Proof of Corollary~\ref{cor:groundArg=CGround(M)}}

First, we note that
every constrained argument of a CABA framework is a constrained instance of itself
(see Definition~\ref{def:instance}) by 
taking  the substitution $\vartheta$ to be the empty substitution $\varepsilon$ and  
$D$ to be the empty set of constraints. 

Then, we prove the following closure property of the 
set  $\CArg$ of all constrained arguments.
This property follows from the fact that
the composition of two
substitutions is a substitution.
\begin{proposition}\label{prop:instance-closure}
\red{For any CABA framework $F_c$,} every  constrained instance of a constrained argument in $\CArg$ is 
a constrained argument in $\CArg$.
\end{proposition}
\begin{proof}
    Let $\alpha'\!=\!\argur{C'\cup A'}{\sent'}{\ruleset}$ be 
a constrained instance of a constrained argument 
$\alpha\!=\!\argur{C\cup A}{\sent}{R}\in \CArg$.
By Definition~\ref{def:instance} there exists a 
substitution~$\vartheta'$, and a
set $D'\!\subseteq\constr$ of 
constraints such that: 
$(i)$~$C'\!=\!(C\vartheta')\cup D'$, $A'\!=\!A\vartheta'$,
$s'\!=\!s\vartheta'$, and $(ii)$~$C'$ 
is consistent.

By Definition~\ref{def:constrained-arg},
since $\alpha\in \CArg$,
there exists a \emph{tight} constrained argument
$\alpha_0=\argur{C_0\cup A_0}{\sent_0}{\ruleset}\in \TCArg$, 
a substitution~$\vartheta_0$, and a
set $D_0\!\subseteq\constr$ of 
constraints such that: 
$(i)$~$C\!=\!(C_0\vartheta_0)\cup D_0$, $A\!=\!A_0\vartheta_0$,
$s\!=\!s_0\vartheta_0$, and $(ii)$~$C$ 
is consistent.
Therefore the thesis follows by considering the 
\emph{tight} constrained argument
$\alpha_0\in \TCArg$ and by taking in Definition~\ref{def:constrained-arg},
$\vartheta\!=\! \vartheta_0 \vartheta'$ and  $D\!=\!D_0 \vartheta' \cup D'$.
\end{proof}

Now, Corollary~\ref{cor:groundArg=CGround(M)} follows from the following two
equalities:\\  
by Proposition~\ref{prop:equalmostconstr},
$\GrCInst(\MGCArg) = \GrCInst(\TCArg)$, and\\ by Proposition~\ref{prop:instance-closure}, $\GrCInst(\TCArg) = \GroundCArg$. \eop


\subsection*{Proof of Theorem~\ref{theo:GroundCABAvsABA}}
It is enough to observe that for any grounding substitution $\vartheta$,
we have that
    $(\argur{D\cup A}{\sent}{R})\vartheta \in \GroundCArg$ iff $\argur{A\vartheta}{\sent\vartheta}{R'} \in \mathit{Arg}_c$, where $R'=R\vartheta \cup \{d\If~\mid d\in D\vartheta\}$. \eop


\subsection*{Proof of Theorem~\ref{theo:groundequiv}}
 (\emph{Only-if part}) 
     The proof proceeds by cases based on the definition of the equivalence~$\equiv$. 

\noindent \hangindent=4mm
 (1)  Assume that $\alpha_1$ and $\alpha_2$ are
ground constrained arguments such that $\alpha_1\!=\!\alpha_2$ modulo ground constraints. Then we get that $\GrCInst(\alpha_1) =
\{\alpha_1\}\!=\! \{\alpha_2\}\! =\!\GrCInst(\alpha_2)$ modulo ground constraints. 

\noindent \hangindent=4mm
 (2) 
 The proof is immediate by observing that any constrained argument $\alpha$, 
$\GrCInst(\GrCInst(\alpha))=\GrCInst(\alpha)$. 

\noindent \hangindent=4mm
 (3) 
 Assume that 
    $\Gamma, \Delta_1,$ $\Delta_2\!\subseteq\! 
    \CArg$, and  $\Delta_1 \equiv\Delta_2$. In this case, the proof is immediate by observing that, for any $\Delta\! \subseteq\! \CArg$, we have that
$\GrCInst(\Gamma \cup \Delta)=\GrCInst(\Gamma) \cup \GrCInst(\Delta)$.

\noindent
 (\emph{If part}) By Points~(1) 
 and~(3) 
 of Definition~\ref{def:equiv} and  $\GrCInst(\Gamma) = \GrCInst(\Delta)$ \red{modulo ground constraints}, we have that $\GrCInst(\Gamma) \equiv \GrCInst(\Delta)$. 
 Moreover by Points~(2)  
 and~(3) 
 of Definition~\ref{def:equiv}, 
 it holds that $\Gamma \equiv \GrCInst(\Gamma)$ and 
 $\Delta \equiv \GrCInst(\Delta)$. Now, the thesis follows by transitivity
 of equivalence.~\eop

\subsection*{Proof of Proposition~\ref{prop:equivproperties}}


We have to prove the following three points.

\noindent
(1) \emph{Renaming}. 
Let $\alpha\!=\!\argu{C\cup A}{\sent}$. By Theorem~\ref{theo:groundequiv}, we have to prove that 
$\GrCInst(\alpha) =\GrCInst(\alpha\rho)$ modulo ground constraints. By Definition~\ref{def:instance}, 
if $\alpha'\!=\!\argu{(C\vartheta)\cup D\cup A\vartheta}{\sent\vartheta}$ is a 
\red{\emph{ground constrained instance $($via $\vartheta$ and~$D$)}} 
of~$\alpha$, then $\alpha'$  is also a  
\red{\emph{ground constrained instance $($via $\rho^{-1}\vartheta$ and~$D$)}} 
of~$\alpha\rho$. Conversely, if  $\alpha'\!=\!\argu{(C\rho\vartheta)\cup D\cup A\rho\vartheta}{\sent\rho\vartheta}$ is a 
\red{\emph{ground constrained instance}} 
of~$\alpha\rho$ $($via $\vartheta$ and~$D$), then $\alpha'$ is a 
\red{\emph{ground constrained instance}} 
of~$\alpha$ $($via $\rho\vartheta$ and~$D$).

\vspace*{1mm}
\noindent
(2) 
\emph{Generalization of Claims and Assumptions}. 
First, we consider the generalization of claims. Let $\alpha=\argu{C \cup A}{p(\tuple{t})}$ and  
$\beta\!=\!\argu{C \cup \{\mbox{$\tuple{X}=\tuple{t}$}\}\cup A}{p(\tuple{X})}$, where
$\mathit{vars}(\alpha) \cap \mathit{vars}(\tuple{X})=\emptyset$.
 By Theorem~\ref{theo:groundequiv}, it suffices to show that 
$\GrCInst(\alpha) =\GrCInst(\beta)$ modulo ground constraints.
Assume that 
$\alpha'\!=\!\argu{(C\vartheta)\cup D\cup A\vartheta}{p(\tuple{t})\vartheta'}$ is a 
\red{\emph{ground constrained instance}} 
of~$\alpha$ $($via $\vartheta$ and~$D$). Without loss of generality, we may assume that $\mathit{dom}(\vartheta)\!=\!\mathit{vars}(\alpha)$. Then for
$\vartheta'=\{\tuple{X}/\tuple{t}\} \cup \vartheta$, we get that
$\beta'\!=\!\argu{(C \cup \{\tuple{X}=\tuple{t}\})\vartheta'\cup D \cup A\vartheta'}{p(\tuple{X})\vartheta'}$  is a 
\red{\emph{ground constrained instance}} 
of~$\beta$ $($via $\vartheta'$ and~$D$). Now, $\beta'\!=\!\alpha'$ modulo ground constraints.
Conversely, if $\beta'\!=\!\argu{(C \cup 
\{\mbox{$\tuple{X}=\tuple{t}$}\})\sigma\cup D \cup A\sigma} {p(\tuple{X})\sigma}$  
is a \red{\emph{ground constrained instance}} 
of~$\beta$ (via $\sigma$ and~$D$), then $\alpha'\!=\!\argu{(C\sigma)\cup D\cup A\sigma}{p(\tuple{t})\sigma}$ is a 
\red{\emph{ground constrained instance}} 
of~$\alpha$ $($via $\sigma$ and~$D$). (Note that the bindings for $\tuple{X}$ 
have no effect on $\alpha$.) Also in this case, we have
$\alpha'\! =\!\beta'$ modulo ground constraints, as desired. The proof for the 
generalization of claims is completed.

The proof for the generalization of assumptions is analogous and is omitted.

\vspace*{1mm}
\noindent
(3) \emph{Constraint Equivalence}. 
Let $C, D_1,$ and $D_2$ be  sets of consistent
constraints such that $\ctheory \models$ $\forall\big(\exists_{-V}
\big(\bigwedge \{c\!\mid\! c\!\in\! C\}\big) \leftrightarrow 
\big(\exists_{-V}\big(\bigwedge \{c\!\mid\! c\!\in\! D_1\}\big) \vee$\! $\exists_{-V}\big(\bigwedge \{c\!\mid\! c\!\in\! D_2\}\big)\big)\big)$, where
 $V\!=\!\mathit{vars}(\{A,s\})$,
 Moreover, let
 $\alpha=\argu{C \cup A}{\sent}$ and 
$\Sigma= \{\argu{D_1 \cup A}{\sent},$ $\argu{D_2 \cup A}{\sent}\}$.
By Theorem~\ref{theo:groundequiv}, we need to show that 
$\GrCInst(\alpha) =\GrCInst(\Sigma)$ modulo ground constraints.
This equality will be proved by showing the two inclusions.

\vspace*{1mm}
\noindent
(\emph{$\subseteq$ part}) If $\alpha'\! \in\! \GrCInst(\alpha)$, then by 
 Definition~\ref{def:consistentgrounding-final}, there exist $D \subseteq \constr$ and a grounding substitution $\vartheta$ such that  $\alpha'\!=\!\argu{(C\vartheta)\cup D\cup A\vartheta}{s\vartheta}$, $(C\vartheta)\cup D$ is consistent,  and $\mathit{vars}(\alpha') =\emptyset$.

Using the equivalence in $\ctheory$ we have that 
there exist two substitutions 
$\vartheta_1$, $\vartheta_2$ such that for $i\!=\!1,2,$
$\vartheta|_V\!=\!\vartheta_i|_V$, 
$\mathit{vars}(D_i \vartheta_i)\!=\!\emptyset$ and
\\
$\begin{array}{l@{\hspace{2mm}}l}
     \ctheory \models & (\bigwedge \{c\vartheta_1\!\mid\! c\!\in\! D_1 \} \wedge \bigwedge \{d\!\mid\! d\!\in\! D\})\ \vee \\
     & (\bigwedge \{c\vartheta_2\!\mid\! c\!\in\! D_2\}\wedge \bigwedge \{d\!\mid\! d\!\in\! D\}).
\end{array} $\\
Then there exists $k\! \in\! \{1,2\}$ such that $\bigwedge \{c\vartheta_k\!\mid\! c\!\in\! D_k \} \wedge \bigwedge \{d\!\mid\! d\!\in\! D\}$ is consistent. Without loss of generality we can assume that $k\!=\!1$.
By construction, recalling that $\vartheta|_V=\vartheta_1|_V$, we have that 
$\beta\!=\!\argu{(D_1\vartheta_1)\cup D\cup A\vartheta}{s\vartheta} \in \GrCInst(\Sigma)$ and $\alpha'\!=\!\beta$ modulo ground constraints.

\balance
\vspace{1mm}\noindent
(\emph{$\supseteq$\,part})\,Similarly, if $\beta \!\in\! \GrCInst(\Sigma)$,\,then\,there exist $k\! \in\! \{1,\!2\}$, \mbox{$D \! \subseteq\! \constr$}, and a grounding substitution $\vartheta_k$ such that  $\beta\!=\!\argu{(D_k\vartheta_k)\cup D\cup A\vartheta_k}{\sent\vartheta_k}$, $(D_k\vartheta_k)\cup D$ is consistent, and $\mathit{vars}(\beta)\! =\!\emptyset$.
Without loss of generality, we can assume that $k\!=\!1$.
By the equivalence in $\ctheory$ we have that there exists $\vartheta$ such that
$\vartheta|_V=\vartheta_1|_V$,
$\mathit{vars}(C \vartheta)=\emptyset$ and 
$\ctheory \models 
\bigwedge \{c\vartheta\!\mid\! c\!\in\! C\}
\wedge \bigwedge \{d\!\mid\! d\!\in\! D\}.$
By construction, recalling that $\vartheta|_V=\vartheta_1|_V$, we have that  
$\alpha'\!=\!\argu{(C\vartheta)\cup D\cup A\vartheta_1}{s\vartheta_1} \in \GrCInst(\alpha)$ and, by construction, $\alpha'\!=\!\beta$ modulo ground constraints.

Thus, $\GrCInst(\alpha) = \GrCInst(\Sigma)$ 
modulo ground constraints, and this completes 
the proof of Point (3) and also the proof of 
the whole Proposition~\ref{prop:equivproperties}. \eop

\subsection*{Proof of Proposition~\ref{prop:pa}}

In order to prove Proposition~\ref{prop:pa}, we need the following two lemmas.

\begin{lemma}
\label{lem:pnongroundattacks}
Let $\alpha,\!\beta\!\in\!\CArg$, and let
$\Gamma\!=\!\GrCInst(\alpha)$, and
$\Delta\!=\!\GrCInst(\beta)$. 
Then $\alpha$ \emph{partially attacks} $\beta$  iff  
there exist 
$\alpha' \! \in\! \Gamma$ and $\beta' \! \in\! \Delta$ such that  $\alpha'$ partially attacks $\beta'$.
\end{lemma}
\begin{proof}
Without loss of generality and by 
   Point~(2) 
   of Proposition~\ref{prop:equivproperties}, we can assume that $\alpha=\argu{\{c_1,\ldots,c_m\} \cup A_1}{\sent_1}$  
and $\beta=\argu{\{d_1,\ldots,d_n\} \cup A_2}{\sent_2}$, where 
$\sent_1=\overline{\asm}$, for some $\asm \in A_2$.
   
Now, we separately prove the \emph{Only-if part} and the \emph{If part}
of the lemma.

(\emph{Only-if part}) Assume that $\alpha$ partially attacks $\beta$. By Definition~\ref{def:nongroundattacks}, $\ctheory \models \exists\, (\exists_{-\mathit{vars}(\sent_1)}(c_1\wedge\ldots\wedge c_m)\wedge \exists_{-\mathit{vars}(\sent_1)}(d_1\wedge\ldots\wedge d_n))$.
Hence, there exist substitutions $\vartheta_1$ and $\vartheta_2$ such that 
$\mathit{dom}(\vartheta_1)=\mathit{vars}(\{c_1,\ldots, c_m\})$,
$\mathit{dom}(\vartheta_2)=\mathit{vars}(\{d_1,\ldots, d_n\})$,
$\vartheta_1|_{\sent_1}=\vartheta_2|_{\sent_1}$,
\linebreak 
$\mathit{vars}(\{c_1\vartheta_1,\ldots, c_m\vartheta_1\}) 
=\emptyset$,  
$\mathit{vars}(\{d_1\vartheta_2,\ldots,$ $d_n\vartheta_2\}) =\emptyset$
and $\ctheory \models (c_1\wedge\ldots\wedge c_m) \vartheta_1\wedge (d_1\wedge\ldots\wedge d_n)\vartheta_2$. 
Thus, we have that there exist
substitutions 
$\sigma_1$ and $\sigma_2$ such that 
$\sigma_1|_{\sent_1}=\sigma_2|_{\sent_1}$, 
$\alpha'=\alpha\vartheta_1\sigma_1$,
$\beta'=\beta\vartheta_2\sigma_2$,
$\mathit{vars}(\alpha')=\mathit{vars}(\beta')=\emptyset$ and 
$\sent_1\vartheta_1\sigma_1=
\sent_1\vartheta_2\sigma_2$.
By construction, we have  
$\alpha'\in \Gamma$,
$\beta' \in \Delta$ and
$\alpha'$ partially attacks
$\beta'$. This concludes the proof of the \emph{Only-if part}.

(\emph{If part})
Assume that there exist
$\alpha' \in \Gamma$ and
$\beta' \in \Delta$ 
such that 
$\alpha'$ partially attacks $\beta'$. 
 By Definitions~\ref{def:constrained-arg} and~\ref{def:consistentgrounding-final},
$\mathit{vars}(\alpha')=\mathit{vars}(\beta')=\emptyset$. Moreover, there exist  
 substitutions~$\vartheta_1$ and $\vartheta_2$, and the
sets of ground
constraints 
$C'=\{c'_1,\ldots,c'_k\}$, 
$D'=\{d'_1,\ldots,d'_h\}$, such that:
$\alpha'=\argu{\{c_1,\ldots,c_m\}\vartheta_1 \cup C'\cup A_1\vartheta_1}{\sent_1 \vartheta_1}$  
and $\beta'=\argu{\{d_1,\ldots,d_n\} \vartheta_2\cup D'\cup A_2\vartheta_2}{\sent_2\vartheta_2} $ 
with
$\sent_1\vartheta_1=\overline{\asm}\vartheta_2$, for some $\asm \in A_2$, and
\begin{equation}
   \ctheory\!\models\!\bigwedge _{i=1, \ldots, m}\!\{c_i \vartheta_1\}\wedge\! 
   \bigwedge _{i=1, \ldots, k}\!\{c'_i \}\wedge \!
\bigwedge _{i=1, \ldots, n}\!\{d_i\vartheta_2\}\wedge \!
   \bigwedge _{i=1, \ldots, h}\!\{d'_i\}. 
   \label{eq:consistent}
\end{equation}
Since by hypothesis $\sent_1=\overline{\asm}$, we have that $\vartheta_1|_{\sent_1}=\vartheta_2|_{\sent_1}$.

We prove by contradiction that $\alpha$ partially attacks $\beta$.
Assume that $\alpha$ does not partially attack $\beta$. Then, by Definition~\ref{def:nongroundattacks}, we must have
    $\ctheory \not\models \exists\, (\exists_{-\mathit{vars}(\sent_1)}(c_1\wedge\ldots\wedge c_m)\wedge \exists_{-\mathit{vars}(\sent_1)}(d_1\wedge\ldots\wedge d_n))$. Hence, for every pair of substitutions 
   $\sigma_1$ and 
   $\sigma_2$, such that 
   $\mathit{dom}(\sigma_1)\supseteq\mathit{vars}(\{c_1,\ldots, c_m\})$,
$\mathit{dom}(\sigma_2)\supseteq\mathit{vars}(\{d_1,\ldots, d_n\})$,
$\sigma_1|_{\sent_1}=\sigma_2|_{\sent_1}$,
$\mathit{vars}(\{c_1\sigma_1,\ldots, c_m\sigma_1\}) = \mathit{vars}(\{d_1\sigma_2,\ldots, d_n\sigma_2\}) =\emptyset$
we have  $\ctheory \not \models 
\bigwedge _{i=1, \ldots, m}\{c_i \sigma_1\}\wedge 
 \bigwedge _{i=1, \ldots, n}\{d_i\sigma_2\}$.
But this contradicts (\ref{eq:consistent}), and this completes the proof of the \emph{If part} and the proof of the whole lemma. \end{proof}

\begin{lemma}
\label{lem:smalllemma}
Given the ground arguments $\alpha$, $\alpha'$, 
$\beta$, and $\beta'$\!, if  $\alpha\!=\!\alpha'$ modulo ground constraints, $\beta\!=\!\beta'$ modulo ground constraints, and $\alpha$ partially attacks $\beta$, then $\alpha'$ partially attacks $\beta'$.
\end{lemma}
\begin{proof} Immediate, because ground arguments that
are equal modulo ground constraints differ by consistent ground constraints only.
\end{proof}

Now we prove the \emph{Only-if part} and the 
\emph{If part} of Proposition~\ref{prop:pa}.

\emph{(Only-if part)} Let $\alpha,\beta$ be constrained arguments such that $\alpha$ partially attacks $\beta$. By Point~(2) of Proposition~\ref{prop:equivproperties}, there exist arguments $\alpha' $ and $\beta_1 $
with
$\alpha\!\equiv\!\alpha' $ and $\beta\!\equiv\!\beta_1 $, where 
$\alpha'=\argu{\{c_1,\ldots,c_m\} \cup A_1}{\sent_1}
$  
and $\beta_1=\argu{\{d'_1,\ldots,d'_n\} \cup A'_2\cup \{\asm\}}
{\sent'_2}$, with $s_1\!=\!\overline{\asm}$
and such that $\ctheory \models \exists\, (\exists_{-\mathit{vars}(\sent_1)}(c_1\wedge\ldots\wedge c_m)\wedge \exists_{-\mathit{vars}(\sent_1)}(d'_1\wedge\ldots\wedge d'_n)))$.
Moreover, by Point~(1) of Proposition~\ref{prop:equivproperties}, there exists a renaming $\rho$ that leaves the variables occurring in $\asm$ unchanged and  replaces all other variables in
$\beta_1$ by fresh new variables not occurring in $\mathit{vars}(\alpha')$ such that
$\beta_1 \equiv \beta'$, where 
$\beta'=\beta_1 \rho=\argu{\{d_1,\ldots,d_n\} \cup A_2\cup \{\asm\}}
{\sent_2}$.
By construction, $\mathit{vars}(c_1,\ldots,c_m) \cap
\mathit{vars}(d_1,\ldots,d_n) \subseteq \mathit{vars}(\sent_1)$. 
Then, since $\ctheory \models \exists\, (\exists_{-\mathit{vars}(\sent_1)}(c_1\wedge\ldots\wedge c_m)\wedge \exists_{-\mathit{vars}(\sent_1)}(d'_1\wedge\ldots\wedge d'_n)))$, it follows that 
$\ctheory \models \exists\, ((c_1\wedge\ldots\wedge c_m)\wedge (d_1\wedge\ldots\wedge d_n)))$. Hence
$\{c_1,\ldots,c_m\}\cup \{d_1,\ldots, d_n\}$
is consistent. This concludes the proof of the 
\emph{Only-if part}.

\emph{(If part)}
Assume that there exist
$\alpha'=\argu{C \cup A_1}{\sent_1}$ and $\beta'=\argu{D \cup A_2}{\sent_2}$
such that $\sent_1\!=\!\overline{\asm}$, for some $\asm\!\in\! A_2$, $\mathit{vars}(C) \cup \mathit{vars}(C)\subseteq\mathit{vars}(s_1)$, and $C\cup D$ is consistent.
By Definition~\ref{def:nongroundattacks},
$\alpha'$ partially attacks $\beta'$.
Therefore, by Lemma~\ref{lem:pnongroundattacks} 
there exist $\gamma' \in \!\GrCInst(\alpha')$
and $\delta' \in\!\GrCInst(\beta')$ such that
$\gamma'$ partially attacks $\delta'$.
Since 
$\alpha\!\equiv\!\alpha' $ and $\beta\!\equiv\!\beta'$, 
by Theorem~\ref{theo:groundequiv} and  Lemma~\ref{lem:smalllemma}, 
we have that there exist
$\gamma \in \!\GrCInst(\alpha)$
and $\delta \in \!\GrCInst(\beta)$ such that:
(i)~$\gamma=\gamma'$ modulo ground constraints, (ii)~$\delta=\delta'$ modulo ground constraints, and 
(iii)~$\gamma$ partially attacks $\delta$. Then,  by Lemma~\ref{lem:pnongroundattacks}, $\alpha$~partially attacks $\beta$ and this 
concludes the proof of the \emph{If part}. 

Thus, the proof of Proposition~\ref{prop:pa} is completed.
\eop

\subsection*{Proof of Proposition~\ref{prop:attacksrelation}}

Assume that $\alpha$ fully attacks $\beta$. 
By Definition~\ref{def:nongroundattacks} and Point~(2) of Proposition~\ref{prop:equivproperties}, without loss of generality, we can assume that $\alpha=\argur{\{c_1,\ldots,c_m\} \cup A_1}{\sent_1}{\ruleset_1}$  
and $\beta=\argur{\{d_1,\ldots,d_n\} \cup A_2}{\sent_2}{\ruleset_2}$ are arguments in $\CArg$ such that 
$\sent_1=\overline{\asm}$ for some $\asm \in A_2$.

By the definition of an argument, the set $\{d_1, \ldots, d_n\}$ is consistent. Therefore, we have:
$\ctheory \models \exists (d_1 \wedge \ldots \wedge d_n).$
Since $\alpha$ fully attacks $\beta$, it follows that:\\
$\ctheory \models \forall\, \big((d_1 \wedge \ldots \wedge d_n) \rightarrow \exists_{-\mathit{vars}(\sent_1)}(c_1 \wedge \ldots \wedge c_m)\big).$

Thus, we get:
\\
$\ctheory \models \exists\, \big((d_1 \wedge \ldots \wedge d_n) \wedge \exists_{-\mathit{vars}(\sent_1)}(c_1 \wedge \ldots \wedge c_m)\big).$

Now, since $FV(\exists_{-\mathit{vars}(\sent_1)}(c_1 \wedge \ldots \wedge c_m)) \subseteq \mathit{vars}(\sent_1)$, we can rewrite the expression as:
\\
$\ctheory \models \exists\, \big(\exists_{-\mathit{vars}(\sent_1)}(d_1 \wedge \ldots \wedge d_n) \wedge \exists_{-\mathit{vars}(\sent_1)}(c_1 \wedge \ldots \wedge c_m)\big).$
Hence, the thesis follows. \eop

\subsection*{Proof of Theorem~\ref{theo:nongroundattacks}}
(i)
   Without loss of generality and by 
   Point~(2) 
   of Proposition~\ref{prop:equivproperties}, we can assume that $\alpha=\argur{\{c_1,\ldots,c_m\} \cup A_1}{\sent_1}{\ruleset_1}$  
and $\beta=\argur{\{d_1,\ldots,d_n\} \cup A_2}{\sent_2}{\ruleset_2}$ are arguments in $\CArg$ such that 
$\sent_1=\overline{\asm}$ for some $\asm \in A_2$.
   
 Now we separately prove the \emph{Only-if part} and the \emph{If part} of this Point~$(i)$.
   
(\emph{Only-if part}) Assume that $\alpha$ fully attacks $\beta$ and let $\beta' \in \Delta$. By Definition~\ref{def:nongroundattacks}, $\ctheory \models \forall \, ((d_1\wedge\ldots\wedge d_n) \rightarrow \exists_{-\mathit{vars}(\sent_1)}(c_1\wedge\ldots\wedge c_m))$.
By Definition~\ref{def:consistentgrounding-final}, there exist
$C' \subseteq \constr$ and a substitution $\vartheta$ such that $\beta'=\argur{\{d_1,\ldots,d_n\}\vartheta  \cup C'\cup A_2\vartheta}{\sent_2\vartheta}{\ruleset_2}$, 
$\{d_1,\ldots,d_n\}\vartheta   \cup C'$ is consistent and 
$\mathit{vars}(\beta')=\emptyset$.
Without loss of generality, we can assume that 
$\mathit{dom}(\vartheta)=\mathit{vars}(\beta)$ and therefore 
$\mathit{vars}(\mathit{rng}(\vartheta))=\emptyset$.
Therefore, by Definition~\ref{def:nongroundattacks},
we get:

$\ctheory \models \exists_{-\mathit{vars}(\sent_1)}
((c_1\wedge\ldots\wedge c_m) \vartheta|_{\sent_1})$. 
\hfill(1)

Let $\tuple{V}$ be a tuple of variables such that $\mathit{vars}(\tuple{V})\!=\!\mathit{vars}(\alpha)\setminus
\mathit{vars}(\sent_1)$.
From (1), we have that
there exists a tuple of ground terms
$\tuple{t_g}$ such that 
$\ctheory \models (c_1\wedge\ldots\wedge c_m) \vartheta'$,
where $\vartheta'\! =\!\vartheta |_{\sent_1}
\{\tuple{V}/\tuple{t_g}\}$ and $\mathit{vars}(\alpha \vartheta')\!=\!\emptyset$. By construction $\alpha'\!=\!\alpha \vartheta'\! \in\! \Gamma$, $\sent_1 \vartheta'=\sent_1 \vartheta$ and 
$\sent_1 \vartheta=\overline{\asm}\vartheta$. 
By definition, $\alpha'$ attacks $\beta'$ in $F_c$. Therefore, by Theorem~\ref{theo:GroundCABAvsABA}, $\alpha'$  also attacks $\beta'$ in $\mathit{Ground}(F_c)$, and this completes the proof.

(\emph{If part}) We assume that for each  
    $\beta' \in \Delta$ there exists 
    $\alpha' \in \Gamma$ such that 
    $\alpha'$ attacks $\beta'$ in $\mathit{Ground}(F_c)$. We prove by contradiction that $\alpha$ fully attacks $\beta$.
    Assume that $\alpha$ does not fully attack $\beta$. Therefore, by Definition~\ref{def:nongroundattacks}, 
    $\ctheory \not \models \forall \, ((d_1\wedge\ldots\wedge d_n) \rightarrow \exists_{-\mathit{vars}(\sent_1)}(c_1\wedge\ldots\wedge c_m))$ and by definition of an argument, we have that $\ctheory \models \exists (d_1\wedge\ldots\wedge d_n)$. Therefore there exists a substitution 
   $\vartheta$, $\ctheory \models (d_1\wedge\ldots\wedge d_n) \vartheta$, 
   $\mathit{vars}((d_1\wedge\ldots\wedge d_n) \vartheta)=\emptyset$ and $\ctheory \not \models \exists_{-\mathit{vars}(\sent_1)}((c_1\wedge\ldots\wedge c_m)\vartheta|_{\sent_1} ) $. 
   Without loss of generality, we can assume that $\mathit{dom}(\vartheta)=\mathit{vars}(\beta)$ and 
   $\mathit{vars}(\beta \vartheta)=\emptyset$. Therefore 
   $\beta \vartheta \in \Delta$.
   
   Now we have a contradiction, since $\ctheory \not \models \exists_{-\mathit{vars}(\sent_1)}((c_1\wedge\ldots\wedge c_m)\vartheta|_{\sent_1} ) $ and therefore for each substitution $\sigma$ such that 
   $\sigma|_{\sent_1}=\vartheta|_{\sent_1}$ and $\mathit{vars}(\alpha \sigma)=\emptyset$, we have that $\alpha \sigma \not \in \Gamma$, since 
   $\ctheory\not \models (c_1\wedge\ldots\wedge c_m)\sigma$. Then, for $\beta'=\beta \vartheta \in \Delta$ there is no $\alpha ' \in \Gamma$ such that 
   $\alpha'$ attacks $\beta'$ in $F_c$.  By Theorem~\ref{theo:GroundCABAvsABA} for 
    $\beta' =\beta \vartheta\in \Delta$ there is no
    $\alpha' \in \Gamma$ such that 
    $\alpha'$ attacks $\beta'$ in $\mathit{Ground}(F_c)$, which contradicts the hypothesis. 
    
$(ii)$ The proof of this point follows from
Lemma~\ref{lem:pnongroundattacks} and 
 Theorem~\ref{theo:GroundCABAvsABA}.

\subsection*{Proof of Theorem~\ref{prop:non-ground-sem} }

Recall that $\mathit{Ground}(F_c)$ denotes the ABA framework derived from $F_c$, as specified in Theorem~\ref{thm:grounding}. Furthermore, we denote by $\mathit{Arg}_c$ the set of all arguments in $\mathit{Ground}(F_c)$, and by
$\Att=\{(\alpha,\beta) \in \mathit{Arg}_c \times \mathit{Arg}_c \mid \alpha$ attacks $\beta\}$
the set of all attacks in $\mathit{Ground}(F_c)$, as introduced in Section~\ref{sec:background}.

Point~(1). We prove the \emph{If part} and the \emph{Only-if part} separately.
        
 (\emph{Only-if part}) The proof is by contradiction.
Assume that $\Sigma$ is conflict-free and $\Sigma$ 
is not NGCF. By definition,
$\exists \alpha_1, \alpha_2 \in \Sigma$, such that $\alpha_1$  partially attacks $\alpha_2$. 
Then by Theorem~\ref{theo:nongroundattacks}, there exist
$\alpha'_1,\alpha'_2 \in \GrCInst(\Sigma)$ such that 
 $\alpha'_1$ attacks $\alpha'_2$ in $\mathit{Ground}(F_c)$
 and therefore $\GrCInst(\Sigma)$ is not a conflict-free extension in $\mathit{Ground}(F_c)$.
Then, by applying Definition~\ref{def:extbyground}, we reach a contradiction with the initial assumption that $\Sigma$ is a conflict-free extension.

(\emph{If part}) The proof is by contradiction.
Assume that $\Sigma$ 
is NGCF and $\Sigma$ is not conflict-free.  By Definition \ref{def:extbyground}, $\GrCInst(\Sigma)$ is not a conflict-free extension in $\mathit{Ground}(F_c)$.   This means that for $i=1,2,$ there exists
$\beta_i \in \GrCInst(\Sigma)$  such that  
 $\beta_1$ attacks $\beta_2$ in $\mathit{Ground}(F_c)$.
 By Theorem \ref{theo:nongroundattacks}, there exist
 $\alpha_1,\alpha_2 \in \Sigma$ such that $\alpha_1$ partially attacks $\alpha_2$ in $F_c$
 and therefore $\Sigma$ 
is not NGCF, contradicting the initial assumption.


Point~(2).  By Definition~\ref{def:extbyground}, $\Sigma$ is stable 
in $F_c$ iff $\GrCInst(\Sigma)$ is stable in $\mathit{Ground}(F_c)$. Then, by definition of stable sets, $\Sigma$ is stable in $F_c$ iff
(i) $\GrCInst(\Sigma)$ is conflict free in $\mathit{Ground}(F_c)$, and (ii) 
$\GrCInst(\Sigma) \cup S' = \mathit{Arg}_c$ modulo ground constraints, where 
$S'=\{\beta \!\in \!\mathit{Arg}_c \mid \exists \alpha \!\in \!\GrCInst(\Sigma)$ such that $(\alpha,\beta) \!\in \!\Att\}$.

By Point(1), (i) holds iff 
$\Sigma$ is NGCF in $F_c$. Now, observe that by Theorems~\ref{theo:nongroundattacks} and \ref{theo:GroundCABAvsABA} and since $S'$ is a set of ground arguments,
$S'=\{\beta \!\in \!\GroundCArg \mid \exists \alpha \!\in \!\Sigma$ such that $\alpha$ fully attacks $\beta $ in $F_c\}$. Therefore, by Point~(i) of Theorem \ref{theo:nongroundattacks}, $S' =  \GrCInst(\FAtt(\Sigma))$.
Then, (ii) holds iff $\GrCInst(\Sigma) \cup \GrCInst(\FAtt(\Sigma)) = \GroundCArg$ modulo ground constraints and then, by Theorem~\ref{theo:groundequiv}, iff $(\Sigma \cup \FAtt(\Sigma)) \equiv CArg$. This completes the proof. \eop

\subsection*{Proof of Proposition~\ref{prop:generalizationcommon}}
In order to prove Proposition~\ref{prop:generalizationcommon}, we need the following lemma.

\begin{lemma}\label{lem:aggiungo}
Let $\alpha
=
\argu{\{c_1,\ldots,c_m,\ \tuple{X}=\tuple{s}\} \cup A \cup \{q(\tuple{X})\}}{p(\tuple{t})}
$ be a constrained argument, where 
 $\mathit{vars}(\tuple{X}) \cap
\bigl(
\mathit{vars}(\{c_1,\ldots,c_m\})
\cup \mathit{vars}(\tuple{s})
\cup \mathit{vars}(A)
\cup \mathit{vars}(p(\tuple{t}))
\bigr)
=
\emptyset$. 
Let $\tuple{Y}$ be a tuple of fresh variables of the same arity as
$\tuple{X}$, such that 
$\mathit{vars}(\tuple{Y}) \cap \mathit{vars}(\alpha )=\emptyset$.
 Then $\{\alpha\} \equiv \{\alpha'\}$, where $
 \alpha'=\argu{\{c_1,\ldots,c_m, \tuple{X}=\tuple{s}, \tuple{Y}=\tuple{s}\} \cup A  \cup \{q(\tuple{X}), q(\tuple{Y})\}}{p(\tuple{t})}$.
\end{lemma}

\begin{proof} 
By Theorem~\ref{theo:groundequiv}, it is sufficient to show that
$\GrCInst(\alpha)=\GrCInst(\alpha')$ modulo ground constraints. 

First, we show that $\GrCInst(\alpha)\subseteq \GrCInst(\alpha')$ modulo ground constraints.
Let $\beta \in \GrCInst(\alpha)$.
By Definitions~\ref{def:constrained-arg} and
\ref{def:consistentgrounding-final}, there exist a grounding substitution~$\vartheta$
and a set $D \subseteq \constr$ of ground constraints such that
$\beta\!=\!\argu{\{c_1,\ldots,c_m,\ \tuple{X}=\tuple{s}\} \vartheta \cup D \cup 
(A \cup \{q(\tuple{X})\}) \vartheta }{p(\tuple{t}) \vartheta}$ is a ground constrained instance of $\alpha$ (i.e., 
$\mathit{vars}(\beta)\!=\!\emptyset$).

Since $\mathit{vars}(\tuple{Y}) \cap \mathit{vars}(\alpha)\!=\!\emptyset$, consider the
substitution 
$\vartheta'\!=\! \{\tuple{Y}/\tuple{s}\vartheta\}\cup \vartheta|_{\alpha}$.
Then
$\beta'\!=\!\argu{\{c_1,\ldots,c_m,\ \tuple{X}\!=\!\tuple{s}, \ \tuple{Y}\!=\!\tuple{s}\} \vartheta' \cup D \cup 
(A \cup \{q(\tuple{X}),q(\tuple{Y})\}) \vartheta' }{p(\tuple{t}) \vartheta'}$ is a constrained instance of $\alpha'$ such that
$\mathit{vars}(\beta')\!=\!\emptyset$, and therefore
$\beta'\! \in\! \GrCInst(\alpha')$.
Moreover, by construction, $\beta\!=\!\beta'$ modulo ground constraints.

Second,  we show that $\GrCInst(\alpha)\supseteq \GrCInst(\alpha')$ modulo ground constraints.
Assume that $\beta' \in \GrCInst(\alpha')$.
By Definitions~\ref{def:constrained-arg} and
\ref{def:consistentgrounding-final}, there exist a substitution~$\vartheta$
and a set $D \subseteq \constr$ of constraints such that
$\beta'\!=\!\argu{\{c_1,\ldots,c_m,\ \tuple{X}\!=\!\tuple{s}, 
\ \tuple{Y}\!=\!\tuple{s},\} \vartheta \cup D \cup 
(A \cup \{q(\tuple{X}), q(\tuple{Y})\}) \vartheta }{p(\tuple{t}) \vartheta}$ is a constrained instance of $\alpha'$ and
$\mathit{vars}(\beta')\!=\!\emptyset$.
Since $\mathit{vars}(\alpha)\!\subseteq \!\mathit{vars}(\alpha')$, it follows that
$\beta\!=\!\argu{\{c_1,\ldots,c_m,\ \tuple{X}=\tuple{s}\} \vartheta \cup D \cup 
(A \cup \{q(\tuple{X})\}) \vartheta }{p(\tuple{t}) \vartheta}$ is a constrained instance of $\alpha$ such that
$\mathit{vars}(\beta)\!=\!\emptyset$, and therefore
$\beta \in \GrCInst(\alpha)$.
Moreover, by construction, $\beta\!=\!\beta'$ modulo ground constraints.

Thus, we get that 
$\GrCInst(\alpha)\!=\!\GrCInst(\alpha')$ modulo ground 
constraints. This completes the proof of  Lemma~\ref{lem:aggiungo}.
\end{proof}

Now we prove Proposition~\ref{prop:generalizationcommon}

(\emph{If part}) Straightforward by Theorem \ref{theo:groundequiv}.

(\emph{Only-if part})   Let us consider the constrained arguments 
$\alpha=\argu{\{c_1,\ldots,c_m\} \cup A}{p(\tuple{t})}$ and
$\beta=\argu{\{d_1,\ldots,d_n\} \cup B}{p(\tuple{u})}.$
Let us assume that $\alpha$ and $\beta$ have common
constrained instances. Thus, there exist the 
constrained arguments 
$\gamma\!\in\!\GrCInst(\alpha)$ and
$\gamma'\!\in\!\GrCInst(\beta)$ such that
$\gamma\!=\!\gamma'$ modulo ground constraints. 
Starting from $\gamma$ and 
$\gamma'$, we can construct the constrained 
arguments $\alpha'$ and $\beta'$ 
as required by Proposition~\ref{prop:generalizationcommon}, in various steps
as we now indicate.

\vspace*{1mm}
\noindent
\emph{Step}~(1). By Point~(1) 
   of Proposition~\ref{prop:equivproperties} there exist two renamings $\rho_1$ and $\rho_2$ such that 
       $\alpha_1=(\argu{\{c_1,\ldots,c_m\}\cup A}{p(\tuple{t})})\rho_1$,  $\beta_1=(\argu{\{d_1,\ldots,d_n\} \cup B}{p(\tuple{u}))\rho_2} 
\in \CArg$ and 
$\{\alpha\}\equiv\{\alpha_1\}$, 
$\{\beta\}\equiv\{\beta_1\}$ and $\mathit{vars}(\alpha_1) \cap \mathit{vars}(\beta_1)\!=\!\emptyset$. 
By Theorem~\ref{theo:groundequiv}, by 
Definitions~\ref{def:constrained-arg} 
and~\ref{def:consistentgrounding-final}, since
$\mathit{vars}(\alpha_1)
\cap \mathit{vars}(\beta_1)\!=\!\emptyset$, there exist two 
consistent sets~$C_1$ and $C_2$ of ground constraints and 
a grounding substitution~$\vartheta$ such that:
    
\makebox[7.5mm][r]{$\gamma$\ }$ = \argu{\{c_1,\ldots,c_m\}\rho_1\vartheta \cup C_1 \cup A\rho_1\vartheta}{p(\tuple{t})\rho_1\vartheta}$,  

\makebox[7.5mm][r]{$\gamma'$}$ =\argu{\{d_1,\ldots,d_n\}\rho_2\vartheta \cup C_2 \cup B\rho_2\vartheta}{p(\tuple{u})\rho_2\vartheta}$,

\noindent
{$\gamma\!=\!\gamma'$ modulo ground constraints,} $\mathit{vars}(\gamma)\!=\!\mathit{vars}(\gamma')\!=\!\emptyset$, 
$\ctheory \models (c_1\wedge \ldots\wedge c_m )\rho_1\vartheta \cup C_1$,
$\ctheory \models (d_1\wedge \ldots\wedge d_n )\rho_2\vartheta \cup C_2$, 
{$\ctheory\!\models\! A\rho_1\vartheta \leftrightarrow B\rho_2\vartheta$, and
$\ctheory \models p(\tuple{t})\rho_1\vartheta \leftrightarrow p(\tuple{u})\rho_2\vartheta$.} Without loss of generality, by Point~(1)
of Proposition~\ref{prop:equivproperties} 
   we can assume that
   $\mathit{dom}(\vartheta) = \mathit{vars}(\alpha_1) \cup \mathit{vars}(\beta_1)$.
   
\vspace*{1mm}
\noindent
\emph{Step}~(2). Let $\tuple{X}$ be a tuple of variables of the same arity of $\tuple{t}$ (which is the same of $\tuple{u}$) such that
       $\mathit{vars}(\tuple{X}) \cap \mathit{vars}(\alpha_1)\! = \!\mathit{vars}(\tuple{X}) \cap \mathit{vars}(\beta_1)\!=\!\emptyset$.    
       
 Moreover, let 
     \blue{$\alpha_2=\argu{\{c_1,\ldots,c_m\}\rho_1 \cup 
       \{\tuple{X}=\tuple{t}\rho_1\}\cup A\rho_1}{p(\tuple{X})}$,}
      \blue{ $\beta_2=\argu{\{d_1,\ldots,d_n\}\rho_2\cup 
       \{\tuple{X}=\tuple{u}\rho_2\} \cup B\rho_2}{p(\tuple{X})} 
\in \CArg$.} 
Then, by Point~(2) 
of Proposition~\ref{prop:equivproperties},  
we have that
$\{\alpha\}\!\equiv\!\{\alpha_2\}$, 
$\{\beta\}\!\equiv\!\{\beta_2\}$ and, by construction, 
$\{c_1,\ldots,c_m\}\rho_1 \cup 
       \{\tuple{X}\!=\!\tuple{t}\rho_1\}\cup \{d_1,\ldots,d_n\}\rho_2\cup 
       \{\tuple{X}\!=\!\tuple{u}\rho_2\}$ is consistent.

\vspace*{1mm}
\noindent
\emph{Step}~(3).
Let $q(\tuple{s_g}) \in A\rho_1\vartheta$.
Thus, by construction, $\tuple{s_g}$ is a tuple of ground terms. 
{Let $A_{q(\tuple{s_g})}$ be the set 
$\{q(\tuple{s}) \!\in\! A\rho_1 \mid q(\tuple{s})\vartheta \!\in\! A\rho_1\vartheta$  and
$q(\tuple{s})\vartheta \!=\! q(\tuple{s_g})$ modulo ground constraints$\}$.
Analogously, let 
$B_{q(\tuple{s_g})}$ be the set $\{q(\tuple{s})\! \in\! B\rho_2 \mid q(\tuple{s})\vartheta
\!\in\! B\rho_2\vartheta$ and 
$q(\tuple{s})\vartheta \!=\! q(\tuple{s_g})$ modulo ground constraints$\}.$
(Recall also that $\ctheory \models A\rho_1\vartheta
\leftrightarrow B\rho_2\vartheta$.)}
Let $h$ and $k$ be the cardinalities of
$A_{q(\tuple{s_g})}$ and $B_{q(\tuple{s_g})}$, respectively, and assume
without loss of generality, that $h \geq k$.
For each $i \in \{1,\ldots,h\}$, let
$\tuple{X}_{i,q(\tuple{s_g})}$ be a tuple of 
fresh, new variables of the same
arity as $\tuple{s_g}$.

Let
\blue{$\alpha_3 = \argu{C_3 \cup A_3}{p(\tuple{X})}$}
be the constrained argument obtained from $\alpha_2$ by replacing each
atom $q(\tuple{s})\! \in\! A_{q(\tuple{s_g})}$ by a distinct atom
$q(\tuple{X}_{i,q(\tuple{s_g})})$, for $i\! \in\! \{1,\ldots,h\}$ (so that no
two atoms in $A_{q(\tuple{s_g})}$ are mapped to the same atom), and by
adding the corresponding constraint
\mbox{$\{\tuple{X}_{i,q(\tuple{s_g})}\!=\!\tuple{s}\}$}
to
$\{c_1,\ldots,c_m\}\rho_1 \cup \{\tuple{X}\!=\!\tuple{t}\rho_1\}$.

Analogously, let
\blue{$\beta'_2 = \argu{D'_2 \cup B'_2}{p(\tuple{X})}$}
be the constrained argument obtained from $\beta_2$ by replacing each
atom $q(\tuple{s})\! \in\! B_{q(\tuple{s_g})}$ by a distinct atom
$q(\tuple{X}_{i,q(\tuple{s_g})})$, for $i \in \{1,\ldots,k\}$, and by
adding the corresponding constraint
$\{\tuple{X}_{i,q(\tuple{s_g})}\!=\!\tuple{s}\}$
to
$\{d_1,\ldots,d_n\}\rho_2 \cup \{\tuple{X}\!=\!\tuple{u}\rho_2\}$.

Finally, let
\blue{$\beta_3 = \argu{D_3 \cup B_3}{p(\tuple{X})}$}
be obtained from $\beta'_2$ by adding, 
whenever $k\!<\!h$, for each
$i \in \{k\!+\!1,\ldots,h\}$, the atom
$q(\tuple{X}_{i,q(\tuple{s_g})})$ to $B'_2$ and the corresponding constraint
$\{\tuple{X}_{i,q(\tuple{s_g})}\!=\! \tuple{s}\}$ to $D'_2$, where
$q(\tuple{s})$ is any atom in $B_{q(\tuple{s_g})}$.


\vspace*{1mm}
\noindent
\emph{Step}~(4).
We apply the construction described above at Step~(3) successively to all atoms
$q(\tuple{s_g})\! \in\! A\rho_1\vartheta$
{by considering also the atoms occurring in $B\rho_2\vartheta$ that are equal to 
$q(\tuple{s_g})$ modulo ground constraints.}
At each application, the construction is performed on the constrained
arguments resulting from the previous application, so that the effects
of the construction accumulate.
Once all atoms in $A\rho_1\vartheta$ have been considered, we denote by
$\alpha' = \argu{C' \cup A'}{p(\tuple{X})}$
and
$\beta' = \argu{D' \cup B'}{p(\tuple{X})}$
the constrained arguments obtained at the end of the process.

\vspace*{1mm}
\noindent
\emph{Step}~(5) By construction $A'=B'$. Moreover, by Point~(2) 
of Proposition~\ref{prop:equivproperties}, by Lemma~\ref{lem:aggiungo} and since $\equiv$ is an 
equivalence relation, we have that
$(i)~\{\alpha\}\equiv\{\alpha'\}$, 
$(ii)~\{\beta\}\equiv\{\beta'\}$. 
$(iii)~\mathit{vars}(C')\cap \mathit{vars}(D') \subseteq \mathit{vars}(\{A',p(\tuple{X})\})$,
and $(iv)~C'\cup D'$ is consistent. This concludes the proof of the \emph{Only-if part} and of the whole Proposition~\ref{prop:generalizationcommon}. \eop

\subsection*{Proof of Theorem~\ref{theo:fully-att}}

(1) Directly from Point (1) of Theorem~\ref{prop:non-ground-sem} and the definition of non-overlapping set of constrained arguments.

\vspace*{1mm}
\noindent
(2) (\emph{If part}) Suppose that $\Sigma$ is conflict-free and $\forall \alpha\! \in\! \Sigma$, if $\exists \beta\!\in\! \Delta\!\setminus\! \Sigma$ such that $\beta$ fully attacks $\alpha$, then $\exists \gamma\!\in\! \Sigma$ such that $\gamma$ fully attacks~$\beta$. 
Take $\mu\!\in\! \mathit{Arg}_c$ such that there exists $\lambda\!\in\! \GrCInst(\Sigma)$ and
$\mu$ attacks $\lambda$ in $\mathit{Ground}(F_c)$. By Theorem~\ref{theo:GroundCABAvsABA}, since $\Delta\equiv \CArg$ and by Theorem~\ref{theo:groundequiv}, we have that $\mu$ belongs to $\GrCInst(\Delta\!\setminus\!\Sigma)$.
Therefore, there exists $\beta\!\in\!\Delta\!\setminus\!\Sigma$ such that $\mu$ is a ground instance of $\beta$ and there exists 
$\alpha\! \in\! \Sigma$ such that $\lambda$ is a ground instance of $\alpha$ and, 
by Theorem~\ref{theo:nongroundattacks}, $\beta$ partially attacks $\alpha$ in $F_c$. Since by hypothesis 
$\Delta$ is non-overlapping, $\beta$ fully attacks $\alpha$. 
Thus, by hypothesis, there exists $\gamma\!\in\! \Sigma$ such that $\gamma$ fully attacks $\beta$, and hence, by Theorem~\ref{theo:nongroundattacks}, there exists $\nu\in \GrCInst(\Sigma)$ such that~$\nu$ attacks $\mu$ in $\mathit{Ground}(F_c)$. Thus, $\Sigma$ is admissible.

(\emph{Only-if part}) Suppose that $\Sigma$ is admissible and let 
$\alpha \!\in\! \Sigma$ and $\beta\!\in\! \Delta\!\setminus\!\Sigma$ 
such that 
$\beta$ fully attacks $\alpha$. By Theorem~\ref{theo:nongroundattacks}, there exist
$\lambda \!\in\! \GrCInst(\alpha) \subseteq \GrCInst(\Sigma)$ and 
$\mu \!\in\! \GrCInst(\beta) \subseteq \GrCInst(\Delta\!\setminus\!\Sigma)$ such that $\mu$ attacks 
$\lambda$ in $\mathit{Ground}(F_c)$. Since $\Sigma$ is admissible, 
there exists $\nu \!\in\! \GrCInst(\Sigma)$, 
 such that $\nu$ attacks $\mu$ in $\mathit{Ground}(F_c)$. Let $\gamma \!\in\! \Sigma$, such that $\nu \!\in\! \GrCInst(\gamma)$.
 By Theorem~\ref{theo:nongroundattacks},
$\gamma$ partially attacks $\beta$. Since $\Delta$ is not-overlapping, $\gamma$ fully attacks $\beta$. This completes the proof of the \emph{Only-if} part.

\vspace*{1mm}
(3.1) Suppose that $\Sigma$ is conflict-free and $\forall\beta\!\in\!\Delta\setminus \Sigma$, $\exists \alpha\!\in\!\Sigma$ such that $\alpha$ fully attacks $\beta$. Take $\mu\!\in\!\mathit{Arg}_c\!\setminus\!\GrCInst(\Sigma).$ Since $\Delta\equiv \CArg$, by Theorems~\ref{theo:GroundCABAvsABA} and~\ref{theo:groundequiv}, $\mu\in \GrCInst(\Delta\!\setminus\!\Sigma)$.
 Then, there exists  $\beta\!\in\!\Delta\!\setminus\! \Sigma$ such that $\mu$ is a ground instance of $\beta$. 
 Thus, by hypothesis, there exists  $\alpha\!\in\! \Sigma$ such that $\alpha$ fully attacks $\beta$, and hence, by Theorem~\ref{theo:nongroundattacks}, there exists  $\lambda\!\in\! \GrCInst(\Sigma)$ such that $\lambda$ attacks $\mu$ in $\mathit{Ground}(F_c)$. Thus, by Definition \ref{def:extbyground}, $\Sigma$ is a stable extension in $F_c$.

\vspace*{1mm}
(3.2) Assume that $\Delta$ is instance-disjoint and $\Sigma$ is stable. Then, by Definition~\ref{def:extbyground}, $\Sigma$ is conflict-free. 
Now, let us consider $\beta\!\in\!\Delta\!\setminus\! \Sigma$. 
Since $\Delta$ is 
instance-disjoint, there exists 
$\mu\!\in\!\GrCInst(\Delta\!\setminus\!\Sigma)$ such that 
$\mu\!\in\!\GrCInst(\beta)$ and there are no $\beta'\!\in\!\Sigma$ and $\mu'\!\in\!\GrCInst(\beta')$ such that $\mu\!=\!\mu'$ modulo ground constraints.
Since $\Sigma$ is a stable extension, there exists $\lambda\!\in\!\GrCInst(\Sigma)$ such that $\lambda$ attacks $\mu$ in  $\mathit{Ground}(F_c)$. 
Therefore
 $\lambda$ is a ground constrained instance of an argument $\alpha\!\in\! \Sigma$ and by Theorem~\ref{theo:nongroundattacks}, $\alpha$ partially attacks $\beta$. Since $\Delta$ is non-overlapping, we have that $\alpha$ fully attacks~$\beta$. This completes the proof of Point~(3.2).~\eop

\subsection*{Proof of Proposition~\ref{prop:cs}}

In order to prove Proposition~\ref{prop:cs},  we need the following lemma.

\begin{lemma}\label{lm:split}
Let $\ctheory$ be closed under negation. For any constraints $c_1,\ldots,c_m\in\constr$, there exist conjunctions  $e_1,\ldots,e_r$ of constraints such that $\ctheory\models \forall (\neg (c_1\wedge\ldots\wedge c_m) \leftrightarrow (e_1\vee\ldots\vee e_r))$ and Points~$(2)$
and~$(3)$
of Definition~$\ref{def:cs}$ hold.
\end{lemma}
\begin{proof}
By De Morgan's laws and closure under negation, there exist constraints $d_1,\ldots,d_n$ such that $\ctheory\models \forall (\neg (c_1\wedge\ldots\wedge c_m) \leftrightarrow (d_1\vee\ldots\vee d_n))$. Now, every disjunction of constraints is equivalent in $\ctheory$ to a disjunction of conjunctions of constraints 
where each disjunct is consistent and each pair of distinct disjuncts are mutually exclusive. Indeed, suppose that 

$D_1$: $d_1\vee\ldots\vee d_k\vee d_{k+1}\vee \ldots \vee d_n$ 

\noindent
is a disjunction of conjunctions of constraints such that, for $i,j\!=\!1,\ldots,k$, we have that $d_i$ is consistent and $d_i,d_j$, with $i\!\neq\! j$, are mutually exclusive. The following disjunction is equivalent to $D_1$ in $\ctheory$:

$D_2$: $(d_1\wedge \neg d_{k+1})\vee\ldots\vee (d_k\wedge \neg d_{k+1})\vee d_{k+1}\vee \ldots \vee d_n$. 

\noindent
By closure under negation, there exist constraints $f_1,\ldots,f_n$ such that~$\ctheory\models \forall (\neg d_{k+1} \leftrightarrow (f_1\vee\ldots\vee f_n))$, where the $f_i$'s are all consistent and pairwise mutually exclusive.
Thus, by using De Morgan's laws, in $\ctheory$ we have that $D_2$ is equivalent to:



\makebox[8mm]{$D_3$: }$(d_1\wedge f_1)\vee\ldots\vee (d_1\wedge f_n)\vee\ldots\vee (d_k\wedge f_1) \vee\ldots\vee (d_k\wedge f_n)\vee$  

\makebox[8mm]{}$d_{k+1}\vee \ldots \vee d_n$.

\noindent
Now, all disjuncts in $D_3$ up to and including $d_{k+1}$ are pairwise mutually exclusive because: (i)~the $d_i$'s, 
for $1\!\leq\!i\!\leq\!k$, are pairwise mutual exclusive, (ii)~the $f_i$'s are pairwise mutually exclusive, and (iii)~each $f_i$ implies $\neg d_{k+1}$.
Then, every inconsistent disjunct can be dropped from $D_3$, still obtaining an equivalent disjunction.
Thus, by induction on the number of disjuncts, we get that $d_1\vee\ldots\vee d_n$ is equivalent in $\ctheory$ to a disjunction $e_1\vee\ldots\vee e_n$ of conjunctions of constraints,
such that every disjunct $e_i$ is consistent and every two disjuncts $e_i,e_j$, with $i\!\neq\! j$, are mutually exclusive.
\end{proof}

Now we prove Proposition~\ref{prop:cs}.\\
Let $C=\{c_1,\ldots,c_m\}$ and $D=\{d_1,\ldots,d_n\}$, with $V=vars(C)\cap vars(D)$.
By closure under existential quantification, there exist conjunctions 
$f_1,\ldots,f_s$ of constraints  such that:

\vspace*{1mm}

$\ctheory \models \forall(\exists_{-V}(c_1\wedge\ldots\wedge c_m) \leftrightarrow (f_1\vee \ldots\vee f_s))$

\vspace*{1mm}

\noindent
and thus, by De Morgan's laws, we get: 

\vspace*{1mm}

$\ctheory \models \forall(\neg \exists_{-V}(c_1\wedge\ldots\wedge c_m) \leftrightarrow (\neg f_1\wedge \ldots\wedge \neg f_s))$

\vspace*{1mm}

\noindent
By Lemma~\ref{lm:split}, for $i=1,\ldots,s$, there exist conjunctions of constraints $g_{i,1},\ldots,g_{i,k_i}$, such that: 
(C1) $\ctheory\models \forall (\neg f_i \leftrightarrow (g_{i,1}\vee\ldots\vee g_{i,k_i}))$, and
(C2) 
$g_{i,j_1},g_{i,j_2}$ are mutually exclusive, for $j_1,j_2=1,\ldots,k_i$, with $j_1\neq j_2$. 
Thus,

\vspace*{1mm}

$\ctheory \models \forall(\neg \exists_{-V}(c_1\wedge\ldots\wedge c_m) \wedge d_1 \wedge \ldots \wedge d_n)$ $~~$\emph{iff}
\hfill \{ by (C1) \}

\vspace*{1mm}

$\ctheory \models \forall((\bigwedge_{i=1}^s\bigvee_{j_i=1}^{k_i}g_{i,j_i}) \wedge d_1 \wedge \ldots \wedge d_n)$ $~~$\emph{iff}
\hfill \{by distributivity\}

\vspace*{1mm}
$\ctheory \models \forall((\bigvee_{j_1=1}^{k_1}\ldots  \bigvee_{j_s=1}^{k_s} \bigwedge_{i=1}^s g_{i,j_i}) \wedge d_1 \wedge \ldots \wedge d_n)$. \hfill $(\dagger1)$

\vspace*{1mm}
\noindent {From $(\dagger1)$,
by a further application of distributivity, we get:}

\vspace*{1mm}
$\ctheory \models \forall((\bigvee_{j_1=1}^{k_1}\ldots  \bigvee_{j_s=1}^{k_s} \bigwedge_{i=1}^s (g_{i,j_i} \wedge d_1 \wedge \ldots \wedge d_n)))$ \hfill $(\dagger2)$

\vspace*{1mm}
\noindent
{Now, in any two disjuncts occurring in $(\dagger2)$,} there exist $g_{i,j_p}$ 
and $g_{i,j_q}$, for some $i\!\in\!\{1,\ldots,s\}$, $j_p\!\in\!
\{1,\ldots,k_p\}$, $j_q\!\in\!\{1,\ldots,k_q\}$, with 
$p,q\!\in\!\{1,\ldots,s\}$ and $j_p\!\neq\! j_q$. Thus, by 
(C2), all pairs of disjuncts are mutually exclusive. Finally, we can select from $(\dagger2)$ 
only consistent disjuncts, say the $r$ conjunctions  $e_1,\ldots,e_r$ of constraints, and hence, for $i\!=\!1,\ldots,r$, we can take $E_i$ to be the set of constraints in the conjunction $e_i$.
\eop

\balance{}
\subsection*{Proof of Proposition~\ref{prop:split}}


\noindent
Point (1.i). Suppose that $\alpha$ and $\beta$ have a common constrained instance.
Then, by Proposition~\ref{prop:generalizationcommon}, 
 $\alpha\equiv\argu{C\cup A}{s}$, $\beta\equiv\argu{D\cup A}{s}$, with 
$\mathit{vars}(C)\cap \mathit{vars}(D) \subseteq \mathit{vars}(\{A,s\})$
and $C\cup D$ consistent. Thus, by Proposition~\ref{prop:cs}, $\textit{split}_{ci}(\alpha,\beta)$ is indeed defined, that is, there exist $\beta_1,\ldots,\beta_r$ such that $\textit{split}_{ci}(\alpha,\beta)= \{\beta_1,\ldots,\beta_r\}$, where 
for $i=1,\ldots,r,$ $\beta_i=\argu{E_i\cup A}{s}$, and $E_i$ satisfies Conditions 
(1)--(3)
of Definition~\ref{def:cs}.
From Condition~(1),
it follows that, for $i=1,\ldots,r,$ $\alpha$ and $\beta_i$ have no common constrained instances.
From Condition~(3),
it follows that, for $i,j\!=\!1,\ldots,r,$ with $i\!\neq\! j,$ $\beta_i$ and $\beta_j$ have no common constrained instances.

\vspace*{2mm}

\noindent
Point (1.ii). From Condition~(1)
of Definition~\ref{def:cs}, 
Proposition~\ref{prop:equivproperties}~(3) 
and Definition~\ref{def:equiv} (3), we get that $(\Delta\!\setminus\!\{\beta\})\cup\{\beta_1,\ldots,\beta_r\} \!\equiv\! \Delta$.

\vspace*{2mm}

\noindent
Point (2.i).
Suppose that $\alpha$ partially attacks $\beta$.
Then, by Proposition~\ref{prop:pa},
$\alpha\equiv\argu{C \cup A_1}{\sent_1}$ and $\beta\equiv\argu{D \cup A_2}{\sent_2}$
such that: $(i)$ $\sent_1\!=\!\overline{\asm}$, for some $\asm\!\in\! A_2$, $(ii)$ $\mathit{vars}(C)\!\cap\! \mathit{vars}(D) \subseteq \mathit{vars}(\sent_1)$, and
$(iii)$ $C\cup D$ is consistent. 

Then, by Proposition~\ref{prop:cs}, $\textit{split}_{pa}(\alpha,\beta)$ is indeed defined, that is, there exist $\beta_0,\beta_1,\ldots,\beta_r$ such that $\textit{split}_{pa}(\alpha,\beta)= \{\beta_0,\beta_1,\ldots,\beta_r\}$, where 
$E_0=C\cup D$ and, for $i\!=\!0,\ldots,r,$ $\beta_i=\argu{E_i\cup A_2}{s_2}$, and $E_i$ satisfies Conditions (1)--(3)
of Definition~\ref{def:cs}.
Let $C\!=\!\{c_1,\ldots,c_m\}$ and $D\!=\!\{d_1,\ldots,d_n\}$. Then, { since we have that $\ctheory\models \forall ((c_1\wedge\ldots\wedge c_m$ $\wedge$ $d_1\wedge\ldots\wedge d_n) \rightarrow \exists_{-\mathit{vars}(s_1)}(c_1\wedge\ldots\wedge c_m))$} and $\beta_0\!=\!\argu{C\cup D \cup A_2}{s_2}$, we get that $\alpha$ fully attacks $\beta_0$.

\vspace{2mm}

\noindent
Point (2.ii).
From Condition~(3)  
of Definition~\ref{def:cs}, we immediately get that, for $i=1,\ldots,r,$ $\alpha$ does not partially attack $\beta_i$.

\vspace{2mm}
\noindent
Point (2.iii).
From Condition~(1)
of Definition~\ref{def:cs}, 
Proposition~\ref{prop:equivproperties}~(3) 
and Definition~\ref{def:equiv} (3), we get that 
$(\Delta\!\setminus\!\{\beta\})\!\cup\!\{\beta_0,\beta_1,\ldots,\beta_r\} 
\!\equiv\! \Delta$. \hfill$\Box$



\section{Some Additional Examples and Results}
In this section we give some additional examples and results that may 
be useful for understanding the definitions, the propositions, 
and the theorems
we have presented in the paper.

\vspace{1mm}

For the CABA framework $\mathit{FA}$ (see
Example~\ref{ex:simple7newMG})
we have that the constrained arguments
$\argu{\{9\!<\!10,\ a(9,0)\}}{p(0)}$ 
and 
$\argu{\{Y\!<\!10,$ $Y\!=\!7,$ $a(Y,0)\}}{p(0)}$
are constrained instances of 
the tight constrained argument
$\argu{\{\mbox{$Y\!<\!10$},$ $a(Y,0)\}}{p(0)}$ which is an instance 
of $\alpha_2$.

\vspace{1mm}


With reference to Definition~\ref{def:consistentgrounding-final}, we have the following proposition about ground constrained instances
of constrained arguments. 
\begin{proposition}
    For any constrained argument $\alpha$, we have that
$\GrCInst(\alpha)\!\neq\! \emptyset$.
\end{proposition}   
\vspace*{-3mm}
\begin{proof}
This proposition follows directly from
the following two facts: 
(i)~constraints in every constrained 
argument are consistent, 
and (ii)~the domain of 
the interpretation of our theory of constraints
is not empty (as it is usually assumed in any 
first order theory).
\end{proof}

In Example~\ref{ex:equivalence} the 
constrained arguments $\alpha_{2,2,1}$ and 
$\alpha_{2,2,2}$  have the  
constraint $Y\!<\!10$ which is redundant, because of  
the constraint $Y\!<\!5$. 
This redundancy is silently assumed in the
presentation of the set $\Gamma$ of 
arguments in 
Example~\ref{ex:attackground}.

\vspace{1mm}

Finally, in Example~\ref{ex:attacks} we have that 
$\alpha_4$ partially attacks $\alpha_1$ 
because
$\ctheory \models \exists X, Y\, (X\!<\!1 \wedge 
\mbox{$Y>0$}  \wedge \mbox{$X<5$} \wedge Y\!>\!3)$.

\vspace{1mm}


%
%

\begin{proposition}
\label{prop:nocommon}
Let us consider $\alpha=\argu{\{c_1,\ldots,c_m\} \cup A}{p(\tuple{t})}$  
and $\beta=\argu{\{d_1,\ldots,d_n\} \cup B}{p(\tuple{u})} 
\!\in\! \CArg$. We have that $\alpha$ and $\beta$  have common 
constrained instances
iff there exist substitutions $\vartheta_1$ and $\vartheta_2$ such that
$\ctheory \models p(\tuple{t})\vartheta_1\!\leftrightarrow\!
p(\tuple{u})\vartheta_2$, $\ctheory \models A\vartheta_1\!\leftrightarrow \!B\vartheta_2$, and 
$\ctheory \models \exists
((c_1\wedge\ldots\wedge c_m)\vartheta_1 \wedge (d_1\wedge\ldots\wedge d_n )\vartheta_2)$.
\end{proposition}
\vspace*{-3mm}
\begin{proof}
Let $\Gamma=\GrCInst(\alpha)$  and $\Delta=\GrCInst(\beta)$.

 (\emph{If part})  
 Assume that there exist two  substitutions $\vartheta_1, \vartheta_2$ such that
 $\ctheory \models p(\tuple{t})\vartheta_1\!\leftrightarrow\! 
p(\tuple{u})\vartheta_2$, $\ctheory \models  A\vartheta_1\!\leftrightarrow\!B\vartheta_2$, and 
$\ctheory \models \exists(
(c_1\wedge\ldots\wedge c_m)\vartheta_1 \wedge (d_1\wedge\ldots\wedge d_n )\vartheta_2)$. Then there exists a 
grounding substitution $\sigma$ such that  $\mathit{vars}(\alpha\vartheta_1\sigma) = \mathit{vars}(\beta\vartheta_2\sigma)=\emptyset$ 
and
$\ctheory \models 
((c_1\wedge\ldots\wedge c_m)\vartheta_1 \wedge (d_1\wedge\ldots\wedge d_n )\vartheta_2)) \sigma$.  
By construction,
$\alpha\vartheta_1\sigma \in \Delta$, $\beta\vartheta_2\sigma \in \Gamma$, 
$\alpha\vartheta_1\sigma =\beta\vartheta_2\sigma$ {modulo ground constraints} and we get the thesis. 
 
 (\emph{Only-if part}) Assume that $\alpha$ and $\beta$ have common constrained instances. 
By Definition
\ref{def:nongrouncommon}
{there exist   
   $\gamma \!\in\! \Gamma$ and 
   $\gamma'\!\in\! \Delta$ such that $\gamma\!=\!\gamma'$ modulo ground constraints.} By Definitions~\ref{def:constrained-arg} and~\ref{def:consistentgrounding-final}, there exist two consistent sets 
    $C_1$ and  $C_2$ of ground constraints and two {grounding} substitutions
   $\vartheta_1$ and $\vartheta_2$ such that:\\ 
\makebox[8mm][r]{$(i)$}~$\gamma\, = \argu{\{c_1,\ldots,c_m\}\vartheta_1 \cup C_1 \cup A\vartheta_1}{p(\tuple{t})\vartheta_1}$, \\ 
\makebox[8mm][r]{$(ii)$}~$\gamma'= \argu{\{d_1,\ldots,d_n\}\vartheta_2 \cup C_2 \cup B\vartheta_2}{p(\tuple{u})\vartheta_2}$,\\
\makebox[8mm][r]{$(iii)$}~$\mathit{vars}(\gamma)\!=\mathit{vars}(\gamma')\!=\!\emptyset$, \\
\makebox[42mm][l]{\makebox[8mm][r]{$(iv)$}~$\ctheory \models (c_1\wedge \ldots\wedge c_m )\vartheta_1\cup C_1$,}\\
\makebox[8mm][r]{$(v)$}~$\ctheory \models (d_1\wedge \ldots\wedge d_n )\vartheta_2\cup C_2$, \\
\makebox[42mm][l]{\makebox[8mm][r]{$(vi)$}~$\ctheory \models A\vartheta_1  \leftrightarrow
B\vartheta_2$,\ \ and}\\
\makebox[8mm][r]{$(vii)$}~$ \ctheory \models p(\tuple{t})\vartheta_1 \leftrightarrow p(\tuple{u})\vartheta_2$.\\
   Now the thesis follows by observing that from $(iii)$--$(v)$, we get 
   $\ctheory \models ((c_1\wedge \ldots\wedge c_m )\vartheta_1 \wedge (d_1\wedge \ldots\wedge d_n )\vartheta_2)$.  
\end{proof}

}
{
}

\end{document}